\documentclass{article}

\usepackage{microtype}
\usepackage{graphicx}
\usepackage{subfigure}
\usepackage{booktabs} 

\usepackage{hyperref}
\usepackage{nccmath}

\usepackage[accepted]{icml2023}

\usepackage{amsmath}
\usepackage{amssymb}
\usepackage{mathtools}
\usepackage{amsthm}

\usepackage[capitalize,noabbrev]{cleveref}

\theoremstyle{plain}
\newtheorem{theorem}{Theorem}[section]
\newtheorem{proposition}[theorem]{Proposition}
\newtheorem{lemma}[theorem]{Lemma}

\theoremstyle{definition}

\newtheorem{assumption}[theorem]{Assumption}
\theoremstyle{remark}

\usepackage{color}
\usepackage{graphics}
\usepackage{epsfig}
\usepackage{subfigure}
\usepackage{natbib}
\usepackage{booktabs}
\usepackage{url}
\usepackage{psfrag}
\usepackage{relsize}
\usepackage{tikz,pgfplots}
\usepackage{caption}

\usepackage{multirow}
\usepackage{threeparttable}
\usepackage{wrapfig}

\DeclareSymbolFont{wideparensymbol}{OMX}{yhex}{m}{n}
\DeclareMathAccent{\wideparen}{\mathord}{wideparensymbol}{"F3} 

\newcommand{\argmin}{\operatornamewithlimits{arg \, min}}
\newcommand{\argmax}{\operatornamewithlimits{arg \, max}}

\usepackage{tikz,pgfplots}
\pgfplotsset{compat=newest}
\pgfplotsset{plot coordinates/math parser=false,trim axis left}
\newlength\figureheight
\newlength\figurewidth

\usepackage{color}
\usepackage[all]{xy}
\allowdisplaybreaks

\begin{document}

\twocolumn[
\icmltitle{Rethinking Weak Supervision in Helping Contrastive Learning}

\icmlsetsymbol{equal}{*}
\begin{icmlauthorlist}
\icmlauthor{Jingyi Cui}{pk1,equal}
\icmlauthor{Weiran Huang}{sj,hw,equal}
\icmlauthor{Yifei Wang}{pk2,equal}
\icmlauthor{Yisen Wang}{pk1,pk3}
\end{icmlauthorlist}

\icmlaffiliation{pk1}{National Key Lab of General Artificial Intelligence, School of Intelligence Science and Technology, Peking University}
\icmlaffiliation{pk2}{School of Mathematical Sciences, Peking University}
\icmlaffiliation{sj}{Qing Yuan Research Institute, Shanghai Jiao Tong University}
\icmlaffiliation{hw}{Huawei Noah’s Ark Lab}
\icmlaffiliation{pk3}{Institute for Artificial Intelligence, Peking University}

\icmlcorrespondingauthor{Yisen Wang}{yisen.wang@pku.edu.cn}

\icmlkeywords{Machine Learning, ICML}

\vskip 0.3in
]

\printAffiliationsAndNotice{\icmlEqualContribution}

\begin{abstract}
Contrastive learning has shown outstanding performances in both supervised and unsupervised learning, and has recently been introduced to solve weakly supervised learning problems such as semi-supervised learning and noisy label learning. Despite the empirical evidence showing that semi-supervised labels improve the representations of contrastive learning, it remains unknown if noisy supervised information can be directly used in training instead of after manual denoising. Therefore, to explore the mechanical differences between semi-supervised and noisy-labeled information in helping contrastive learning, we establish a unified theoretical framework of contrastive learning under weak supervision. Specifically, we investigate the most intuitive paradigm of jointly training supervised and unsupervised contrastive losses. By translating the weakly supervised information into a similarity graph under the framework of spectral clustering based on the posterior probability of weak labels, we establish the downstream classification error bound. We prove that semi-supervised labels improve the downstream error bound whereas noisy labels have limited effects under such a paradigm. Our theoretical findings here provide new insights for the community to rethink the role of weak supervision in helping contrastive learning. 
\end{abstract}

\section{Introduction}

Contrastive learning has shown state-of-the-art empirical performances in unsupervised representation learning \citep{chen2020simple, he2020momentum, chen2021exploring,wang2021residual}.
It learns good representations of high-dimensional observations from a large amount of unlabeled data, by pulling together an anchor and its augmented views in the embedding space. 
On the other hand, supervised contrastive learning \citep{khosla2020supervised} uses same-class examples and their corresponding augmentations as positive labels, and achieves significantly better performance than both the unsupervised contrastive learning and the state-of-the-art classification losses, e.g. cross entropy loss. While accurate supervised signals are not always available, there is a lot of weak supervision accompanying the data. This makes us wonder: \emph{can weak supervision help contrastive learning?}

There are two major types of weak supervision. The first is semi-supervised information, where the supervised labels are only available on a small fraction of samples. The second is noisy-labeled information, where the labels are available but unreliable, i.e. the labels can possibly be wrong. 
Empirical evidence has shown that semi-supervised information can directly be used in positive sample selection to help improve the representations of contrastive learning by jointly training the supervised and unsupervised contrastive losses \citep{assran2020supervision, acharya2022positive}. 
By contrast, for noisy-labeled information, most methodological studies use contrastive learning as a tool to select confident samples based on the learned representations \citep{yao2021jo, ortego2021multi, li2022selective, zhang2022learning}, whereas none of the existing literature demonstrates if noisy label information can help improve the representations of contrastive learning. Therefore, we are wondering \emph{if the conclusion on noisy-labeled information is the same as the semi-supervised information. If not, what are the differences between semi-supervised and noisy-labeled information in helping contrastive learning?}

In this paper, we show an awkward fact that noisy labels have a limited effect on representations of contrastive learning under the joint training paradigm, and the winner of supervised and unsupervised contrastive learning itself serves as an embarrassingly strong baseline. As a preview of results, in Figure \ref{fig::max}, we plot the result of unsupervised contrastive learning trained without labels (SimCLR \citep{chen2020simple}), supervised contrastive learning trained with noisy labels (SupCon \citep{khosla2020supervised}), the winner of SupCon and SimCLR (Max), and joint training of SimCLR and SupCon under label noise (JointTraining), respectively. We conduct experiments under symmetric label noise with noise rates ranging from $0\%$ to $60\%$, stepped by $10\%$. We evaluate the trained representations by linear probing on the clean testing data.
The weight parameter $\theta$ of JointTraining is tuned in $\{0, 0.05, 0.1, 0.2, 0.4, 0.6, 0.8, 1.0\}$, and the best linear probing accuracy under the optimal $\theta$ is reported.
We show that supervised training outperforms unsupervised training when the noise rate is low, and the situation reverses when the noise rate is high.
Nonetheless, the joint training of SimCLR and SupCon (with finely tuned weights) has only limited advantage over the winner of SupCon and SimCLR. Moreover, its performance coincides with that of SupCon when noise rate is $0$, and converges to that of SimCLR when noise rate increases.
This indicates that noisy labels provide limited help to contrastive representation learning under the paradigm of joint training, and that the winner of SupCon and SimCLR itself serves as a strong baseline.

\begin{figure}[!t]
	\begin{center}
		\centerline{\includegraphics[width=\columnwidth]{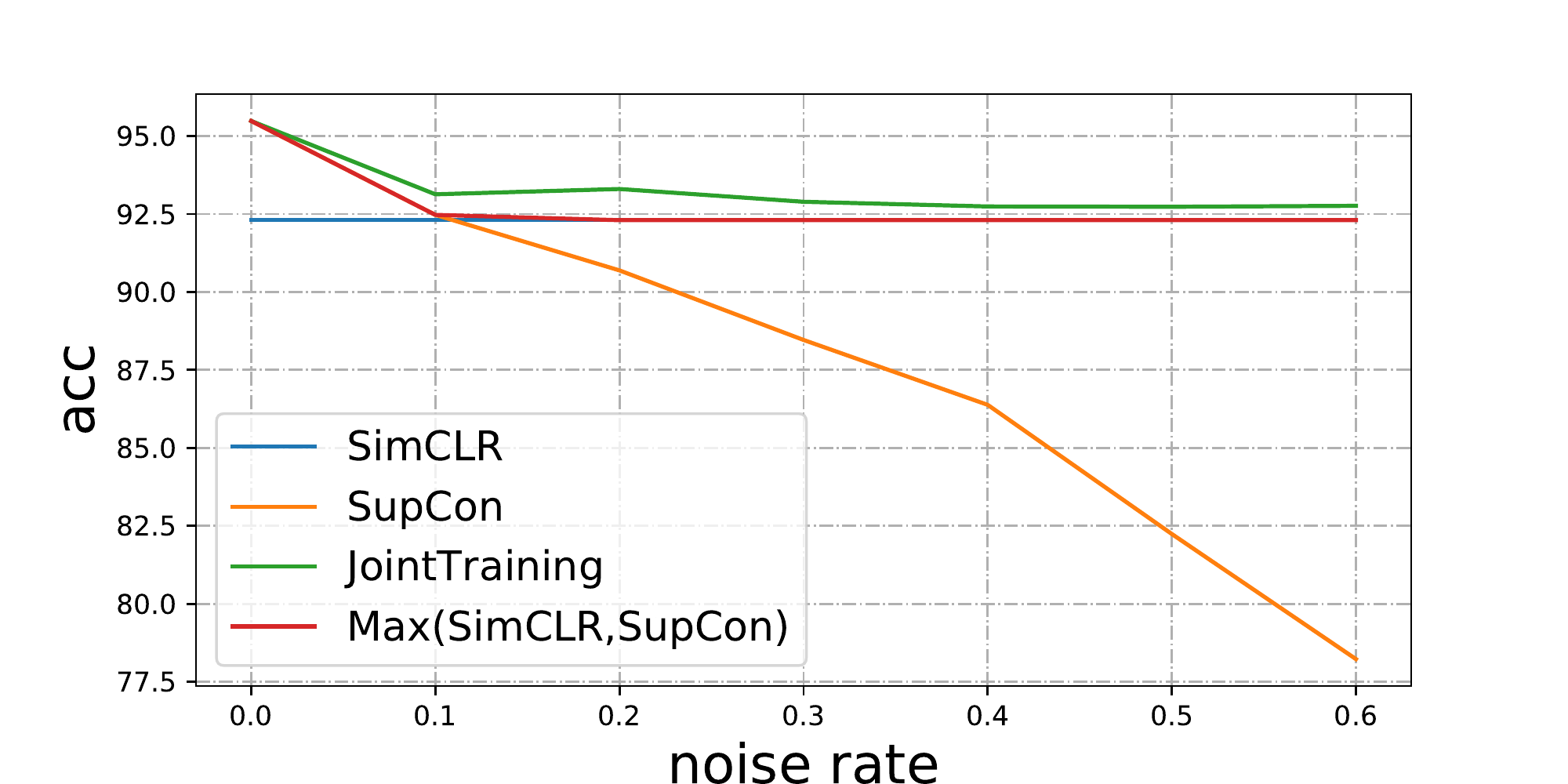}}
        \vskip -0.05in
		\caption{The winner of SupCon and SimCLR serves as a strong baseline of joint training under label noise on the  CIFAR-10 dataset.}
		\label{fig::max}
	\end{center}
	\vskip -0.4in
\end{figure}

To explain the above phenomena in depth and exploit the mechanical differences between semi-supervised and noisy-labeled information in helping contrastive learning, we establish a unified theoretical framework of weakly supervised contrastive learning, applying to both semi-supervised and noisy-labeled settings.
We investigate the joint training of supervised and unsupervised contrastive losses, and take spectral contrastive learning \citep{haochen2021provable} as a performance proxy for (standard) contrastive learning to conduct the theoretical analysis.
Based on the posterior probability of labeled samples, we translate the weakly supervised information (under symmetric label noise assumption) into a similarity graph under the framework of spectral clustering.
This enables us to analyze the effect of the label information on the augmentation graph, and consequently on the derived error bound.
Accordingly, we prove that under the semi-supervised setting, the label information helps improve the downstream error bound, whereas under the noisy-labeled setting, the joint training is no better than the winner of supervised and unsupervised contrastive learning in terms of an error bound.

The contributions of this paper are summarized as follows.
\begin{itemize}
	\item We for the first time establish a theoretical framework for contrastive learning under weak supervision, including noisy label learning and semi-supervised learning.
	\item By formulating the label information into a similarity graph based on the posterior probability of labels, we derive the downstream error bound of jointly trained contrastive learning losses. 
	We prove that semi-supervised labels improve the downstream error bound compared with unsupervised learning, whereas under the noisy-labeled setting, joint training fails to improve the error bound compared with the winner of supervised and unsupervised contrastive learning.
	\item We empirically verify that noisy labels have only limited help to contrastive representation learning under the paradigm of joint training. Thus, complex designs such as label denoising are still required for leveraging noisy labeled information in improving supervised contrastive learning.
\end{itemize}

\section{Related Works}

\paragraph{Theoretical Frameworks of Contrastive Learning.}
The theoretical frameworks of unsupervised contrastive learning can be divided into two major categories. 
The first category is devoted to building the relationship between unsupervised contrastive learning and supervised downstream classification.
\citet{arora2019theoretical} first introduce the concept of latent classes, hypothesizes that semantically similar points are sampled from the same latent class, and proves that the unsupervised contrastive loss serves as an upper bound of downstream supervised learning loss.
\citet{nozawa2021understanding, ash2022investigating, bao2022surrogate} further investigate the effect of negative samples, and establish surrogate bounds for the downstream classification loss that better match the empirical observations on the negative sample size.
However, studies in this category have to assume the existence of supervised latent classes, and that the positive pairs are conditionally independently drawn from the same latent class. This assumption fails to distinguish between supervised and unsupervised contrastive learning, and therefore cannot be used to analyze the weakly supervised setting.

Another major approach is to analyze contrastive learning by modeling the feature similarity.
\citet{haochen2021provable} first introduce the concept of the \textit{augmentation graph} to represent the feature similarity of the augmented samples, and analyzes contrastive learning from the perspective of spectral clustering. 
\citet{shen2022connect} use a stochastic block model to analyze spectral contrastive learning for the problem of unsupervised domain adaption. 
Similarly, \citet{wang2021chaos} propose the concept of \textit{augmentation overlap} to formulate how the positive samples are aligned. 
Moreover, contrastive learning is also understood through other existing theoretical frameworks of unsupervised learning, such as nonlinear independent component analysis \citep{zimmermann2021contrastive}, neighborhood component analysis \citep{ko2022revisiting}, variational autoencoder \citep{aitchison2021infonce}, stochastic neighbor embedding \citep{hu2022your}, geometric analysis of embedding space \citep{huang2023towards}, and message passing \citep{wang2023message}.

In this paper, we follow the second category of contrastive learning approaches, and formulate the weakly supervised information into a similarity graph based on both label and feature information. 

\paragraph{Contrastive Learning for Noisy Label Learning.}
\citet{ghosh2021contrastive} first find that pretraining with contrastive learning improves robustness to label noise through empirical evidence.
Many methodological studies are carried out for noisy label learning with the help of contrastive learning. 
\citet{yao2021jo, ortego2021multi, li2022selective} use representations learned from unsupervised contrastive learning to filter out confident samples from all noisy ones, and in turn use the confident samples to conduct supervised contrastive learning to generate better representations. 
\citet{navaneet2022constrained} introduce additional semantically similar supervision to contrastive representation learning by incorporating nearest neighbors under certain constraints as additional positive samples, which also adapts to noisy label learning.
By contrast, \citet{yan2022noise} follow the idea of negative learning \citep{kim2019nlnl, kim2021joint}, and leverages the negative correlations from the noisy data to avoid same-class negatives in contrastive learning.
\citet{chuang2022robust} propose a robust contrastive loss function inspired by the symmetric losses that are proved to be noise tolerant.
Very recently, \citet{zhang2022learning} use contrastive learning to handle noisy labels of long-tailed data.
For theoretical studies, \citet{cheng2021demystifying} analyze the robustness of cross-entropy with SSL features, and \citet{xue2022investigating} prove the robustness of downstream classifier in contrastive learning.

 \paragraph{Contrastive Learning for Semi-supervised Learning.}
 \citet{lee2022contrastive, yang2022class} use contrastive regularization to enhance the reliability of pseudo-labeling in semi-supervised learning.
 \citet{kim2021selfmatch} introduce a semi-supervised learning method that combines self-supervised contrastive pre-training and semi-supervised fine-tuning based on augmentation consistency regularization.
 \citet{zhang2022semi} use contrastive loss to model pairwise similarities among samples, generates pseudo labels from the cross entropy loss, and in turn calibrates the prediction distribution of the two branches.

To conclude, the existing studies of contrastive learning under weak supervision mainly focus on using contrastive learning as a tool to improve the weakly supervised learning performance, whereas to the best of our knowledge, none of the previous works reveals how weak supervision helps contrastive learning. 
To fill in the blank, in this paper, we establish a theoretical framework for contrastive learning under weak supervision, and show the effects of semi-supervised and noisy-labeled information on the error bounds of contrastive learning.

\section{Preliminaries}

\textbf{Notations.} Suppose that random variables $\bar{X} \in \bar{\mathcal{X}}:= \mathbb{R}^d$, and $Y \in [r]:=\{1,\ldots,r\}$. Let the input natural data $\{(\bar{x}_i,y_i)\}_{i\in[N]}$ be i.i.d.~sampled from the joint distribution $\mathrm{P}(\bar{X},Y)$.
Given a natural data $\bar{x} \in \bar{\mathcal{X}}$, we use $\mathcal{A}(\cdot|\bar{x})$ to denote the distribution of its augmentations and use $\mathcal{X}$ to denote the set of all augmented data, which is assumed to be finite but exponentially large. Denote $n=|\mathcal{X}|$.

\subsection{Spectral Contrastive Learning}

In \citet{haochen2021provable}, an augmentation graph $\mathcal{G}$ is used to describe the distribution of augmented samples, where the edge weight $w_{xx'} := \mathbb{E}_{\bar{x} \sim \bar{\mathcal{P}}} [\mathcal{A}(x|\bar{x}) \mathcal{A}(x'|\bar{x})]$ denotes the marginal probability of generating augmented views $x$ and $x'$ from the same natural data. Due to the total probability mass, $\sum_{x,x'\in\mathcal{X}}w_{xx'}=1$.
The adjacent matrix of the augmentation graph is denoted as $\boldsymbol{A} := (w_{xx'})_{x,x' \in \mathcal{X}} \in \mathbb{R}^{n\times n}$, and the normalized adjacent matrix is denoted as $\bar{\boldsymbol{A}} := D^{-1/2}\boldsymbol{A}D^{-1/2}$, where $D := \mathrm{diag}(w_x)_{x \in \mathcal{X}}$, and $w_x := \sum_{x'\in\mathcal{X}}w_{xx'}$.

In this paper, we consider the spectral contrastive loss $\mathcal{L}(f)$ proposed by \citet{haochen2021provable}, that is, for an embedding function $f : \mathcal{X} \to \mathbb{R}^k$, 
\begin{align}\label{eq::spectralloss}
	-2\cdot \mathbb{E}_{x,x^+}[f(x)^\top f(x^+)] + \mathbb{E}_{x,x'}\Big[\big(f(x)^\top f(x')\big)^2\Big].
\end{align}
Spectral contrastive loss is proved to be equivalent to the matrix factorization loss, i.e. for $F \in \mathbb{R}^{n\times k}:=(u_x)_{x\in\mathcal{X}}$, $u_x := w_x^{1/2}f(x)$, 
\begin{align}\label{eq::mfloss}
	\mathcal{L}_{\mathrm{mf}}(F) := \|\bar{\boldsymbol{A}}-FF^\top\|_F^2 = \mathcal{L}(f) + const.
\end{align}

\subsection{Noisy Label Learning}

Recall that we denote the true label of a given instance $x \in \mathcal{X}$ is $y$. One common assumption of the generation procedure of label noise is as follows.
Given the true labels, the noisy label is randomly flipped to another label $\tilde{y}$ with some probability.
In this paper, we take the widely adopted symmetric label noise assumption as an example.

For notational simplicity, we write the symmetric label noise assumption in matrix form. 
Denote $\boldsymbol{Y} := (\eta_j(x_i))_{i \in [n], j \in [r]}$, $\eta_j(x) = \mathrm{P}(Y=j|x)$, as the posterior probability matrix of the clean label distribution, and denote $\tilde{\boldsymbol{Y}} := (\tilde{\eta}_j(x_i))_{i \in [n], j \in [r]}$, $\tilde{\eta}_j(x) = \mathrm{P}(\tilde{Y}=j|x)$, as the noisy label distribution. 
In Assumption \ref{ass::symnoise}, we assume that the flipping probability is conditional independent of the input data, and that the flipping probability to all other classes is uniformly at random.

\begin{assumption}\label{ass::symnoise}
	For symmetric label noise with noise rate $\gamma \in (0,1)$, we denote the transition matrix $\boldsymbol{T}=(t_{i,j})_{i \in [r], j \in [r]}$, where
	\begin{align}
		t_{i,i}=1-\gamma, \text{ and } t_{i,j}=\frac{\gamma}{r-1}, \text{ for } j\neq i.
	\end{align}	
	Then the noisy label posterior distribution is assumed to be 
	\begin{align}\label{eq::Y_noise}
		\tilde{\boldsymbol{Y}} = \boldsymbol{Y}\boldsymbol{T}.
	\end{align}
\end{assumption}
Under Assumption \ref{ass::symnoise}, $\boldsymbol{T}$ is symmetric. Specifically, when $\gamma=0$, $\boldsymbol{T}$ degenerates to the identity matrix $\boldsymbol{I}_{r \times r}$.
Moreover, to guarantee PAC-learnability, we usually assume the true label is the dominating class, i.e. $\gamma < \frac{r-1}{r}$.

\subsection{Semi-Supervised Learning}

For $j \in [r]$, let $n_j$ be the number of labeled samples of Class $j$. Let $n_L = \sum_{j \in [r]} n_{L,j}$ be the number of all labeled samples, and $n_U$ be the number of unlabeled samples. Obviously, we have $n_L + n_U = n$. Usually, the number of labeled samples is much smaller than that of the unlabeled ones because human annotation is costly and labor-intensive. That is, we can naturally assume $n_L \ll n_U$.

In the following parts of the paper, we analyze the settings of noisy label learning and semi-supervised learning in a unified framework. Without loss of generality, we assume $(x_1, \ldots, x_{n_L})$ is labeled with noise rate $\gamma \in [0,\frac{r-1}{r})$, and denote the corresponding clean and noisy posterior probability matrices as $\boldsymbol{Y}_{L}$ and $\tilde{\boldsymbol{Y}}_{L}$, respectively. Then we have $\tilde{\boldsymbol{Y}}_{L} = \boldsymbol{Y}_{L}\boldsymbol{T}$. 
Specifically, when $\gamma =0$, our analyzing framework degenerates to the standard setting of semi-supervised learning, and when $n_L=n$, our analyzing framework reduces to the standard noisy label learning.

\section{Mathematical Formulations}

We mention that our formulation of ``similarity graph'' is not a distributional assumption on the underlying similarity among data, but to formulate a possible probability of drawing positive samples in contrastive learning that takes both label and feature information into consideration. Specifically, in Sections \ref{sec::graphnoisy} and \ref{sec::graphsemi}, we only discuss the similarity graph induced by the weakly supervised labels and neglected feature similarity. Note that in Section \ref{sec::graphsemi} we investigate the setting of semi-supervised noisy labels, so as to include semi-supervised learning and noisy label learning in a unified framework. Then in Section \ref{sec::graphmix}, we take both label and feature similarity into consideration through the convex combination to describe the joint training loss.

\subsection{Similarity Graph Describing Noisy Labels}\label{sec::graphnoisy}

To leverage the labeled information in the form of a similarity graph, we first consider a simple example where noise rate $\gamma=0$ and the label distribution is deterministic, i.e. for a sample $x$ with true label $y$, the posterior probability $\eta_y(x)=1$ and $\eta_j(x)=0$ for $j \neq y$. In this case, we can naturally assume that in the label similarity graph, the intra-class vertices are fully connected and the inter-class vertices are disconnected. 
That is, $w_{xx'}=1$ if $x$ and $x'$ has the same label and otherwise $w_{xx'}=0$.

Then we consider the more general stochastic label scenario. Recall that for unsupervised spectral contrastive learning, the edge weight $w_{xx'}$ in an augmentation graph $\mathcal{G}$ describes the marginal probability of generating $x$ and $x'$ from the same natural data.
That is, $w_{xx'}$ describes the joint probability of a pair of positive samples. 
Similarly, since the positive samples for supervised contrastive learning \citep{khosla2020supervised} are selected as all same-class samples, we can naturally define the edge weight $w_{xx'}$ as the probability of two views $x$ and $x'$ generating from the same class, i.e. $w_{xx'} = \sum_{j \in [r]} \eta_j(x)\eta_j(x') $, and therefore $\boldsymbol{A}_L := \boldsymbol{Y}_L\boldsymbol{Y}_L^\top$.
Moreover, we denote $\bar{\boldsymbol{A}}$ as the normalized adjacent matrix. For the simplicity of notations, we consider the case where the data is class-balanced, i.e. $n_1 = \ldots = n_r = n_L/r$. Then we have $\bar{\boldsymbol{A}} = \frac{r}{n_L}\boldsymbol{A}$.

Next, we add label noise to our mathematical formulations. To be specific, when performing supervised contrastive learning based on noisy labeled data, we naturally select positive samples as the samples with the same \textit{noisy} labeled data. 
According to Assumption \ref{ass::symnoise}, we have $\tilde{\boldsymbol{Y}}_L = \boldsymbol{Y}_L\boldsymbol{T}$, where $\boldsymbol{T}$ is symmetric.
Then the adjacent matrix of the similarity graph describing noisy labels is formulated as
\begin{align}
	\boldsymbol{A}^{\star}_L :&= \tilde{\boldsymbol{Y}}_L \tilde{\boldsymbol{Y}}_L^\top = \boldsymbol{Y}_L\boldsymbol{T} (\boldsymbol{Y}_L\boldsymbol{T})^\top 
	= \boldsymbol{Y}_L\boldsymbol{T} \boldsymbol{T}^\top \boldsymbol{Y}_L^\top
	\nonumber \\
	&= \boldsymbol{Y}_L\boldsymbol{T}^2 \boldsymbol{Y}_L^\top.
\end{align}
Similarly, when data is class balanced, we have the normalized adjacent matrix $\bar{\boldsymbol{A}}^{\star}_L = \frac{n_L}{r} \boldsymbol{A}^{\star}_L$.

\subsection{Similarity Graph Describing Semi-Supervised Noisy Labels}\label{sec::graphsemi}

Under the setting of semi-supervised learning, we have no prior knowledge about the label information of the unlabeled samples. Therefore, from the perspective of unsupervised contrastive learning, the unlabeled samples can be viewed as having unique class labels. Therefore, to construct the similarity graph, we attach sample-specific labels to the unlabeled samples. Thus, the posterior probability matrix of unlabeled samples $\boldsymbol{Y}_U$ is an identity matrix $\boldsymbol{I}_{n_U \times n_U}$.
Note that here we only discuss the similarity graph of supervised information, so the feature similarity between samples is not included in the similarity graph.

Combining both labeled and unlabeled samples, the posterior probability matrix of all semi-supervised samples can be denoted as
\begin{align}
	\tilde{\boldsymbol{Y}} = 
	\begin{bmatrix}
		\tilde{\boldsymbol{Y}}_L & \boldsymbol{0} \\
		\boldsymbol{0} & \tilde{\boldsymbol{Y}}_U
	\end{bmatrix}
	=
	\begin{bmatrix}
		\boldsymbol{Y}_L \boldsymbol{T} & \boldsymbol{0} \\
		\boldsymbol{0} & \boldsymbol{I}_{n_U \times n_U}
	\end{bmatrix}.
\end{align}

Therefore, the similarity graph of samples with $n_L$ noisy labels can be denoted as
\begin{align}\label{eq::tildeA}
	\boldsymbol{A}^{\star} 
	= \tilde{\boldsymbol{Y}} \tilde{\boldsymbol{Y}}^\top
	= 
	\begin{bmatrix}
		\boldsymbol{Y}_L \boldsymbol{T}^2 \boldsymbol{Y}_L^\top & \boldsymbol{0} \\
		\boldsymbol{0} & \boldsymbol{I}_{n_U \times n_U}.
	\end{bmatrix}
\end{align}

In Lemma \ref{lem::relation} we present the influence of symmetric label noise with noise rate $\gamma$ on the similarity graph $\boldsymbol{A}^{\star}$. 
\begin{lemma}\label{lem::relation}
	Under Assumption \ref{ass::symnoise}, if the data is class balanced, i.e. $n_1 = \ldots = n_r = \frac{n_L}{r}$, then there holds
	\begin{align}\label{eq::relation}
		\bar{\boldsymbol{A}}^{\star} = 
		\begin{bmatrix}
			\alpha(\gamma) \bar{\boldsymbol{A}}_L + \beta(\gamma) \frac{r}{n_L} \vec{1}_{n_L}\vec{1}_{n_L}^\top & \boldsymbol{0} \\
			\boldsymbol{0} & \boldsymbol{I}_{n_U \times n_U}
		\end{bmatrix},
	\end{align}
	where $\alpha(\gamma) := \big(1-\frac{r}{r-1}\gamma\big)^2$ and $\beta(\gamma) := \frac{\gamma}{r-1}\big(2-\frac{r}{r-1}\gamma\big)$.
\end{lemma}
Note that without label noise, i.e. $\gamma=0$, we have $\alpha(\gamma)=1$ and $\beta(\gamma)=0$. For the sake of simplicity, in the following, we write $\alpha$ and $\beta$ instead of $\alpha(\gamma)$ and $\beta(\gamma)$ when no ambiguity is aroused.

In Lemma \ref{lem::relation}, we show that the effect of symmetric label noise is to add a uniform weight to the edges between all labeled samples. This uniform weight increases the confusion between intra- and inter-class similarities. 
For example, under the deterministic label scenario, we have $\boldsymbol{A}_L = \boldsymbol{I}_{n_L\times n_L}$, and under label noise, the original intra-class similarity is uniformly shrunk from $1$ to $\alpha$ and the inter-class similarity increases from $0$ to $\beta$.
Moreover, as the noise rate $\gamma$ increases, $\alpha$ decreases and $\beta$ increases, which results in severer confusion between the intra- and inter-class similarities. 
Intuitively, as the similarity graph describes the sampling probability of positive pairs, $\alpha$ and $\beta$ measure \textit{how noisy that similarity is}. (The sampling probability of negative samples remains unaffected because they are assumed to be uniformly sampled regardless of labels.) 
Then under label noise, the probability of same-class samples being selected as positive pairs reduces from 1 to $\alpha + \frac{r}{n_L}\beta$, and the probability of different-class samples being selected as positive pairs raises from $0$ to $\frac{r}{n_L}\beta$.

Note that our mathematical formulations can also be extended to generalized label noise assumptions of the noise transition matrix $\boldsymbol{T}$ (other than symmetric label noise). Specifically, under generalized assumptions, the term $\frac{r}{n_L} \vec{1}_{n_L}\vec{1}_{n_L}^\top$ in \eqref{eq::relation} will become a real symmetric matrix which depends on the specific form of the label noise assumption.

\subsection{Similarity Graph Describing Joint Training}\label{sec::graphmix}

Recall that we want to investigate the effect of noisy labels in the joint training of supervised and unsupervised contrastive losses.
Specifically, for $\theta \in (0,1)$, we have the loss for joint training as
\begin{align}\label{eq::jt}
	\mathcal{L}_{\mathrm{JointTraining}} := (1-\theta)\mathcal{L}_{\mathrm{unsup}}+ \theta \mathcal{L}_{\mathrm{sup}},
\end{align}
where in $\mathcal{L}_{\mathrm{unsup}}$, the positive samples are selected as augmentations of the anchor sample, and in $\mathcal{L}_{\mathrm{sup}}$, the positive samples are selected as (possibly noisy) same-class samples.
The specific algorithms of joint training for noisy label learning and semi-supervised learning are shown in Appendix \ref{app::alg}.

According to \eqref{eq::mfloss}, the spectral contrastive loss is equivalent to the matrix factorization loss. 
Therefore, we can write \eqref{eq::jt} as a convex combination of matrix factorization losses with similarity graphs under label and feature information, i.e.
\begin{align}\label{eq::jointmf}
	(1-\theta)\|\bar{\boldsymbol{A}}_0 - FF^\top\|_F^2 + \theta \|\bar{\boldsymbol{A}}^{\star} - FF^\top\|_F^2,
\end{align}
where we denote $\boldsymbol{A}_0$ as the augmentation graph of arbitrary unlabeled samples describing feature information, and $\boldsymbol{A}^{\star}$ is denoted in \eqref{eq::tildeA} describing the similarity graph induced by the semi-supervised noisy labeled (augmented) samples.

Note that \eqref{eq::jointmf} can be rewritten as 
\begin{align}\label{eq::jointA}
	\|((1-\theta) \bar{\boldsymbol{A}}_0 + \theta \bar{\tilde{\boldsymbol{A}}}) - FF^\top\|_F^2 + c_0(\theta),
\end{align}
where $c_0(\theta):=(1-\theta)\|\bar{\boldsymbol{A}}_0\|_F^2 + \theta \|\bar{\tilde{\boldsymbol{A}}}\|_F^2 - \|(1-\theta) \bar{\boldsymbol{A}}_0 + \theta \bar{\tilde{\boldsymbol{A}}}\|_F^2$ is independent of $F$.
As $c_0(\theta)$ does not affect the training procedure, optimizing \eqref{eq::jointmf} and \eqref{eq::jointA} results in the same optimal $F$.
Therefore, in the following, to analyze the joint training loss \eqref{eq::jt}, we investigate the properties of the mixed similarity graph 
\begin{align}\label{eq::graphmix}
	\boldsymbol{A}_{\theta, \gamma, n_L}
	:= (1-\theta)\bar{\boldsymbol{A}}_0 + \theta\bar{\boldsymbol{A}}^{\star}.
\end{align}

\section{Theoretical Results}

In this section, we first compute eigenvalues of the similarity graph induced by both label and feature information, which plays a key role in deriving the error bound of contrastive learning in Section \ref{sec::eigen}. 
Then in Section \ref{sec::bound_semi}, we show that (clean) label information in the semi-supervised setting can help improve the error bound, whereas in Section \ref{sec::bound_noisy}, we prove that the joint training of supervised and unsupervised contrastive learning fails to improve the error bound compared with purely supervised or purely unsupervised contrastive learning.

\subsection{Eigenvalues of Similarity Graph Describing Joint Training}\label{sec::eigen}

We first compute the eigenvalues of the similarity graph describing the weak labels (without describing feature information).
\begin{proposition}\label{lem::tilde_lambda}
	For arbitrary $\boldsymbol{Y}$, assume that the labeled data is class-balanced, i.e. $\sum_{i \in [n_L]} \eta_j(x_i)=n_L/r$ for $j \in [r]$. Assume that the eigenvalues of $\bar{\boldsymbol{A}}_L$ are $\mu_1, \ldots, \mu_n$ (in descending order). Then under Assumption \ref{ass::symnoise}, the eigenvalues of $\bar{\boldsymbol{A}}^{\star}$ are 
	\begin{align}
		& \tilde{\mu}_1 = \ldots = \tilde{\mu}_{n_U+1} = 1, \\
		& \tilde{\mu}_j = \mu_j \alpha, \text{ for } j = n_U+2, \ldots, n,
	\end{align}
    where $\alpha = \mu_j \Big(1-\frac{r}{r-1}\gamma\Big)^2$.
\end{proposition}

In Proposition \ref{lem::tilde_lambda}, we show that the eigenvalues of $\bar{\boldsymbol{A}}^{\star}$ rely on the eigenvalues of $\bar{\boldsymbol{A}}$ and consequently rely on the posterior probabilities of clean labels. Specifically, if the true label has a higher posterior probability, i.e. $\max_{j \in [r]}\mathrm{P}(Y=j|x)$ is larger, then the eigenvalues of $\bar{\boldsymbol{A}}$ are larger.
On the other hand, the existence of label noise uniformly shrinks the eigenvalues of $\bar{\boldsymbol{A}}^{\star}$ except for the largest ones, and larger noise rate $\gamma$ results in smaller $\alpha$ and thus leads to smaller eigenvalues of $\bar{\boldsymbol{A}}^{\star}$.
Moreover, the number of largest eigenvalues is decided by the number of unlabeled samples.

Note that $\mathrm{rank}(\boldsymbol{A}^{\star}) \leq \mathrm{rank}(\boldsymbol{Y}_L) + n_U \leq n_U + r$, and therefore we have $\tilde{\mu}_{n_U+r+1} = \ldots = \tilde{\mu}_n = 0$.
Specifically, under the deterministic label scenario, we have $\mu_{n_U+2} = \ldots = \mu_{n_U+r}=1$. Then the eigenvalues of $\bar{\boldsymbol{A}}^{\star}$ become
\begin{align}
	& \tilde{\mu}_1 = \ldots = \tilde{\mu}_{n_U+1} = 1, \\
	& \tilde{\mu}_{n_U+2} = \ldots = \tilde{\mu}_{n_U+r} = \alpha = \Big(1-\frac{r}{r-1}\gamma\Big)^2, \\
	& \tilde{\mu}_{n_U+r+1} = \ldots = \tilde{\mu}_n = 0.
\end{align}

Then in the following proposition, we discuss the eigenvalues of the mixed similarity graph $\boldsymbol{A}_{\theta, \gamma, n_L}$ describing both weak labels and feature information. 
\begin{proposition}\label{lem::eigenmix}
		Denote $\lambda_1,\ldots,\lambda_n$ as the eigenvalues of $\boldsymbol{A}_{\theta,\gamma,n_L}$. Then given the eigenvalues of $\bar{\boldsymbol{A}}_0$, i.e. $\nu_1, \ldots, \nu_n$ and the eigenvalues of $\bar{\boldsymbol{A}}_L$, i.e. $\mu_1, \ldots, \mu_{n_L}$(in descending order), under the deterministic scenarios, when $k \leq n_U$, there holds
\begin{align*}
	&\max\medmath{\big\{\theta+(1-\theta)\nu_{n_L+k},  (1-\theta)\nu_{k+1},
	\theta\alpha+(1-\theta)\nu_{n_L+k-r+1}\big\}}
	\nonumber\\
	&\leq \medmath{
	\lambda_{k+1}
	\leq \theta +  (1-\theta)\nu_{k+1}},
\end{align*}
	and for $k \geq n_U+r$,
	\begin{align*}
		&(1-\theta)\nu_{k+1} \leq
		\lambda_{k+1}
		\nonumber\\
		&\leq\medmath{\min\{\theta + (1-\theta)\nu_{k+1}, \theta\alpha + (1-\theta)\nu_{k-n_U}, 
		(1-\theta)\nu_{k+1-r-n_U}\}.}
	\end{align*}
\end{proposition}
In Proposition \ref{lem::eigenmix}, we derive both the upper and lower bounds for the $k+1$-th largest eigenvalue of $\boldsymbol{A}_{\theta, \gamma, n_L}$.
We see that under the deterministic scenario, the upper bound of the $k+1$-th largest eigenvalue of $\boldsymbol{A}_{\theta, \gamma, n_L}$ depends on at most three specific eigenvalues of the unsupervised augmentation graph $\boldsymbol{A}_0$. 
The value of $\lambda_{k+1}$ is also affected by the weighting parameter $\theta$. However, the specific dependence relies on the relative magnitudes of $\nu_{k+1}$, $\nu_{k-n_U}$, and $\nu_{k+1-n_U-r}$.
A perhaps somehow anti-intuitive conclusion is that when $k$ is smaller than $n_U$, the upper bound of $\lambda_{k+1}$ is unaffected by the noise rate. 
Similarly, we notice that for $k \geq n_U+r$, the lower bound of $\lambda_{k+1}$ is unaffected by the noise rate.

\subsection{Error Bound of Joint Training under Semi-supervised Setting}\label{sec::bound_semi}

Recall that the goal of contrastive representation learning is to learn an embedding function $f : \mathcal{X} \to \mathbb{R}^k$. The quality of the learned embedding is often evaluated through linear evaluation. To be specific, denote $B \in \mathbb{R}^{k\times r}$ as the weights of the downstream linear classifier, and the linear predictor is denoted as $g_{f, B}(\bar{x}) = \argmax_{i \in [r]} \mathrm{P}_{x\sim\mathcal{A}(\cdot|\bar{x})}(g_{f, B}(x)=i)$. 
In this paper, we focus on analyzing the error bound of the best possible downstream linear classifier $g_{f^*_{\mathrm{pop}}, B^*}$, where $f^*_{\mathrm{pop}} \in \argmin_{f: \mathcal{X}\to\mathbb{R}^{k}}$ is the minimizer of the population spectral contrastive loss $\mathcal{L}(f)$ defined in \eqref{eq::spectralloss}, and $B^*$ is the optimal weight for the downstream linear classifier.

Following \citet{haochen2021provable}, we assume that the labels are recoverable from augmentations, i.e. we assume there exists a classifier $g$ that can predict $y(x)$ given $x$ with error at most $\delta \in (0,1)$.
\begin{assumption}\label{ass::superror}
	Assume that for some $\delta_u, \delta_s > 0$, there holds
	\begin{align}
		\mathbb{E}_{\bar{x}\sim\mathcal{P}_{\bar{X}}, x\sim\mathcal{A}(\cdot|\bar{x})} \boldsymbol{1}[\hat{y}(x_i) \neq y(\bar{x})] \leq \delta_u.
	\end{align}
	and
	\begin{align}\label{eq::errorunif}
		\frac{1}{n_L} \sum_{i \in[n_L]} \sum_{\ell \in [r]} \eta_{\ell}(x_i) \boldsymbol{1}[\hat{y}(x_i) \neq \ell]
		\leq \delta_s.
	\end{align}
\end{assumption}

Compared with Assumption 3.5 in \citet{haochen2021provable}, Assumption \ref{ass::superror} additionally assumes the recoverable of labels taking expectation under the posterior probability distribution. 
Intuitively, $\delta_u$ represents the error under unsupervised learning, and $\delta_s$ represents the error under supervised learning (with clean posterior distributions). Therefore, it is reasonable to assume that $\delta_s \leq \delta_u$.
Note that Assumption \ref{ass::superror} is a minor revision of the original assumption. The additional assumption \eqref{eq::errorunif} does not change the nature of the original idea of label recovery, and will be used to bound the error term of learning from weakly supervised labels.

Then we derive the error bound of downstream linear evaluation learned by contrastive learning under semi-supervised setting.

\begin{theorem}\label{thm::pop_bound_semi}
	Assume the assumptions in Theorem \ref{thm::pop_bound_noisy} hold.
	Then if 
	\begin{align*}
		\|B^*\|_F 
		&\leq 1/\max\Big\{(1-\theta)\nu_{k},
		\theta+(1-\theta)\nu_{n_L+k-r}\Big\},
	\end{align*}
	we have
	\begin{align}\label{eq::boundsemi}
		\mathcal{E}	
		&\leq 
		\frac{2\big[2\delta_u + [(1+\rho)\delta_s - 2\delta_u]\theta\big]}{1-\theta - (1-\theta)\nu_{k+1}} + 8\delta_u.
	\end{align}
\end{theorem}

By Theorem \ref{thm::pop_bound_semi}, the error bound of linear probing is larger when the label recovery error $\delta_u$ and $\delta_s$ gets larger.

The error bound in Theorem \ref{thm::pop_bound_semi} attains the minimum when $\theta=1$ if $\rho \leq 2\delta_u/\delta_s-1$.
And therefore we have $\mathcal{E} \leq 2(1+\rho)\delta_s + 8\delta_u \leq 2(1+\rho)\delta_u + 8\delta_u$.
Since $\nu_{k+1} \in [0,1]$, then when $\rho \leq (1+\nu_{k+1})/(1-\nu_{k+1})$, 
$2(1+\rho)\delta_u + 8\delta_u \leq \frac{4\delta_u}{1-\nu_{k+1}}+8\delta_u$.
Recall that in \citet{haochen2021provable}, the error bound of purely unsupervised contrastive learning is $\frac{4\delta_u}{1-\nu_{k+1}}+8\delta_u$.
Our result indicates that semi-supervised information improves the error bound compared with purely unsupervised contrastive learning by using all labeled samples. This theoretical point can also be verified by existing experimental research about semi-supervised contrastive learning, e.g. \citet{assran2020supervision}.

\subsection{Error Bound of Joint Training under Noisy-labeled Setting}\label{sec::bound_noisy}

In Theorem \ref{thm::pop_bound_noisy}, we present the error bound of joint training contrastive learning under noisy label setting.

\begin{theorem}\label{thm::pop_bound_noisy}
	For arbitrary $\boldsymbol{Y}$, assume that the labeled data is class-balanced, i.e. $\sum_{i \in [n_L]} \eta_j(x_i)=n_L/r$ for $j \in [r]$. Denote $\nu_1, \ldots, \nu_n$ as the eigenvalues of $\bar{\boldsymbol{A}}_0$ (in descending order). 
	Denote $\mathcal{E} := \mathrm{P}_{\bar{x}\sim \mathcal{P}_{\bar{X}}, x \sim \mathcal{A}(\cdot|\bar{x})} 
	\big(g_{f^*_{\mathrm{pop}}, B^*} (x) \neq y(\bar{x}) \big)$ as the linear evaluation error, where $B^* \in \mathbb{R}^{r\times k}$ with norm $\|B^*\|_F\leq 1/(1-\theta)\nu_{k}$.
	Assume there exists $\rho > 0$, such that $w_i/w_j < \rho$, for $i, j \in [n]$, and $k>r$.
	Then under the deterministic scenario and Assumptions \ref{ass::symnoise} and \ref{ass::superror}, there holds
	\begin{align}\label{eq::boundnoisy}
		\mathcal{E}
		\leq 
		{\small \frac{2\big[2\delta_u + [\alpha(1+\rho)\delta_s - 2\delta_u + (1-\alpha)]\theta\big]}{1-\lambda(\boldsymbol{\nu}; \theta, \alpha)}} + 8\delta_u,
	\end{align}
	where 
	\begin{align}
		\lambda(\boldsymbol{\nu}; \theta, \alpha) &= \min\{\theta + (1-\theta)\nu_{k+1}, \theta\alpha + (1-\theta)\nu_{k},
		\nonumber\\
		& \qquad \quad \ (1-\theta)\nu_{k+1-r}\}.
	\end{align}
	and $\alpha := \big(1-\frac{r}{r-1}\gamma\big)^2$.
\end{theorem}

The bound in Theorem \ref{thm::pop_bound_noisy} gets larger when 
the noise rate $\gamma$ and 
the label recovery error $\delta_u$ and $\delta_s$ gets larger.
It is worth noting that Theorem \ref{thm::pop_bound_noisy} shows that joint training fails to improve the error bound. 
Specifically, since the numerator is positive, this bound is equivalent to the minimum of the following three terms
\begin{align}
	\frac{2\big[2\delta_u + [\alpha(1+\rho)\delta_s - 2\delta_u + (1-\alpha)]\theta\big]}{1-\theta - (1-\theta)\nu_{k+1}} + 8\delta_u,
\end{align}
\begin{align}
	\frac{2\big[2\delta_u + [\alpha(1+\rho)\delta_s - 2\delta_u + (1-\alpha)]\theta\big]}{1 - \theta\alpha - (1-\theta)\nu_{k}} + 8\delta_u,
\end{align}
and
\begin{align}
	\frac{2\big[2\delta_u + [\alpha(1+\rho)\delta_s - 2\delta_u + (1-\alpha)]\theta\big]}{1 - (1-\theta)\nu_{k+1-r}} + 8\delta_u.
\end{align}
Since each of the three terms is a monotonic function with respect to $\theta$, the minimum must be attained at the end of the range of $\theta$, i.e. $\theta=0$ or $\theta=1$. Therefore, the minimum of $\mathcal{E}$ must be attained at $\theta=0$ or $\theta=1$, too.
Specifically, when the noise rate is relatively high, i.e. $\gamma > \gamma_{\mathrm{threshold}} := \frac{r-1}{r}\Big(1-\sqrt{\frac{1-\nu_{k+1}-2\delta_u}{(1-(1+\rho)\delta_s)(1-\nu_{k+1})}}\Big)$, $\mathcal{E}$ achieves its minimum at $\theta=0$, whereas when the noise rate is relatively low, i.e. $\gamma < \gamma_{\mathrm{threshold}}$, $\mathcal{E}$ achieves its minimum at $\theta=1$. 
Either way, joint training of supervised and unsupervised contrastive learning ($\theta \in (0,1)$) does not improve the error bound compared with purely supervised or purely unsupervised contrastive learning.

Then we derive the finite sample bound for the linear probing error for contrastive learning under label noise, which explicitly depends on the training set size $n$. 

\begin{theorem}\label{thm::sample_bound_noisy}
 Denote $\widehat{\mathcal{E}}:=\mathrm{P}_{\bar{x}\sim \mathcal{P}_{\bar{X}}, x \sim \mathcal{A}(\cdot|\bar{x})} \big(g_{\hat{f}, \widehat{B}} (x) \neq y(\bar{x}) \big)$ as the generalization bound for the linear probing error. For any labeling function $\hat{y}: \mathcal{X} \to [r]$, there exists a linear probe $\widehat{B} \in \mathbb{R}^{r\times k}$ such that with probability at least $1 - \varepsilon$ over the randomness of data, we have
	\begin{align}\label{eq::sample_bound_noisy}
		\widehat{\mathcal{E}} 
		&\leq {\small \min_{1\leq k'\leq k} \bigg(\frac{2\big[2\delta_u + [\alpha(1+\rho)\delta_s - 2\delta_u + (1-\alpha)]\theta\big]}{1-\lambda(\boldsymbol{\nu}; \theta, \alpha, k'+1)}}
		\nonumber\\
		&+{\small \frac{4k'\Big[c_1 \cdot \widehat{R}_{n/2}(\mathcal{F}) + c_2 \big(\sqrt{\frac{\log2/\delta}{n}}+\varepsilon\big)\Big]}{\big((1-\theta)\nu_{k'} - \lambda(\boldsymbol{\nu}; \theta, \alpha, k+1)\big)^2} \bigg)+ 8\delta_u,}
	\end{align}
	where $\widehat{R}_{n/2}(\mathcal{F})$ is the maximal possible empirical Rademacher complexity of $\mathcal{F}$ over $n/2$ data, $\lambda(\boldsymbol{\nu}; \theta, \alpha, k+1) = \min\{\theta + (1-\theta)\nu_{k+1}, \theta\alpha + (1-\theta)\nu_{k}, (1-\theta)\nu_{k+1-r}\}$, $\alpha := \big(1-\frac{r}{r-1}\gamma\big)^2$, $c_1 \lesssim k^2\kappa^2+k\kappa$ and $c_2 \lesssim k\kappa^2 + k^2\kappa^4$.
\end{theorem}

Note that in Theorem \ref{thm::sample_bound_noisy}, as the training size $n \to \infty$, the sample error term (second term) approximates $0$, and therefore the bound in Theorem \ref{thm::sample_bound_noisy} degenerates to that in Theorem \ref{thm::pop_bound_noisy}.

For a given sample size $n < \infty$, we observe a trade-off in the choice of $k$. Specifically, as $k$ increases, the approximation error (1st term in \eqref{eq::sample_bound_noisy}) decreases, whereas the sample error (2nd term in \eqref{eq::sample_bound_noisy}) increases. It will lead to the following two cases: 1) when $k$ is small, the approximation error could be very large since $\lambda_{k+1}$ is large; 2) when $k$ is large, the eigen gap $\lambda_{k'}-\lambda_{k+1} \to 0$ since both $\lambda_{k'}$ and $\lambda_{k+1}$ are very small, and accordingly the sample error goes to infinity. This suggests that we should choose a moderate feature dimension $k$, so that both approximation and sample error terms are relatively small.

\subsection{Discussions}\label{sec::discussion}

By comparing Theorems \ref{thm::pop_bound_semi} and \ref{thm::pop_bound_noisy}, we show that clean semi-supervised labels help improve the downstream linear error bound of contrastive representation learning, whereas jointly training supervised and unsupervised contrastive losses fails to improve the error bound under noisy labels. In other words, the winner of supervised and unsupervised contrastive learning itself serves as a strong baseline for noisy label contrastive learning.
This theoretical finding partly explains why the intuitive joint-training method is not investigated by the community, and why complex algorithmic design such as label denoising is a popular approach to leveraging noisy labels in contrastive learning.

For technical contributions, although the theoretical analysis is based on \citet{haochen2021provable}, our analysis is essentially different from existing works in the following aspects: 1) We for the first time establish a theoretical framework for weakly supervised contrastive learning, where we translate the label information into a similarity graph, whereas existing works analyzed pure unsupervised contrastive learning; 2) The main technical difficulty of our analysis is to discuss the eigenvalues of the mixed similarity graph containing both label and feature information (Proposition \ref{lem::eigenmix}), rather than to utilize existing results about self-supervised contrastive learning. 

The joint training of SimCLR and SupCon is also discussed in previous works \citep{islam2021broad,chen2022perfectly}, which focus on the empirical improvement of transfer performances whereas we focus on the theoretical properties of joint training. More details can be found in Appendix \ref{app::compare}.

\section{Experiments}

Recall that it is already empirically verified by \citet{assran2020supervision} that clean semi-supervised labels help improve over unsupervised contrastive learning.
Therefore, in this section, we only empirically verify our theoretical results that noisy labels have limited effects in improving the performance of contrastive learning, and show that the winner of SupCon and SimCLR itself serves as a strong baseline for contrastive learning with noisy labels. 
Based on this, we discuss that complex designs are imperative for improving contrastive representation learning with noisy labels.

\subsection{Experimental Setups}\label{sec::exp_setting}
We conduct numerical comparisons on the CIFAR-10 and TinyImageNet-200 benchmark datasets. 
Because the standard supervised contrastive learning algorithm SupCon \citep{khosla2020supervised} is adapted from the self-supervised contrastive learning framework SimCLR \citep{chen2020simple}, for fair comparisons, we use SupCon as the supervised contrastive loss, and SimCLR as the self-supervised contrastive loss. We argue that we can to a large extent verify the theoretical insights discussed in the previous section, even if the theoretical parts consider the spectral contrastive loss. 
First of all, as shown in the original paper, spectral contrastive loss has comparative empirical performances with respect to that of SimCLR.
Besides, the theoretical evidence can be found in \citet{johnson2022contrastive}, which proves that by interpreting the exponentiated dot product $e^{f(x)^{\top}f(x')}$ as the similarity and treating the exponential and temperature term $\tau$ as part of the model instead of part of the objective, InfoNCE loss and Spectral contrastive loss share the same population minimum. 
That means, by adopting similar kernel deriviations, our work also has the potential to extend to other contrastive losses including standard InfoNCE.

We follow the experimental setting of SimCLR and SupCon. 
Specifically, we use ResNet-50 as the encoder and a 2-layer MLP as the projection head. 
We set the batch size as $1024$. We use $1000$ epochs for training representations.
We use the SGD optimizer with the learning rate $0.5$ decayed at the $700$-th, $800$-th, and $900$-th epochs with a weight decay $0.1$.
We run experiments on 4 NVIDIA Tesla V100 32GB GPUs.
The data augmentations we use are random crop and resize (with random flip), color distortion, and color dropping.
We evaluate the self-supervised learned representation by linear evaluation protocol, where a linear classifier is trained on the top of the encoder, and regard its test accuracy as the performance of the encoder.
$100$ epochs are used for linear probing on the clean data.
The symmetric noisy labels are generated by flipping the labels of a given proportion of training samples uniformly to one of the other class labels.
For the CIFAR-10 dataset, we run experiments with noise rate $\{0.1, 0.2, 0.3, 0.4, 0.5, 0.6\}$ and for TinyImageNet-200, we run experiments with noise rate $\{0.2, 0.4, 0.6, 0.8\}$.

\subsection{Parameter Analysis of $\theta$}

We first conduct an analysis of the parameter $\theta$ for the joint training of SimCLR and SupCon. 
In Figure \ref{fig::theta_cifar}, we plot the optimal $\theta$ for $\mathcal{L}_{\mathrm{JointTraining}}$ on the CIFAR-10 dataset across various noise rates. We show that when the noise rate is relatively high, $\theta$ is relatively small, and $\mathcal{L}_{\mathrm{JointTraining}}$ relies more on unsupervised learning, whereas when the noise rate is relatively low, $\theta$ is relatively large, and $\mathcal{L}_{\mathrm{JointTraining}}$ relies more on supervised learning. 
Moreover, the optimal $\theta$ lies very close to either $0$ or $1$, which indicates that only one loss mainly contributes to the joint training. 
On the other hand, in Figure \ref{fig::theta_img}, we plot the performance changes of $\mathcal{L}_{\mathrm{JointTraining}}$ with respect to the parameter $\theta$ under noise rates $0.4$ and $0.8$ on the TinyImageNet-200 dataset. We show that under both noise rates, as $\theta$ increases, the accuracy of $\mathcal{L}_{\mathrm{JointTraining}}$ decreases. Moreover, the performance drop of $\mathcal{L}_{\mathrm{JointTraining}}$ under high noise rates is more significant than that under low noise rates, indicating that too much noisy label information hurts the performance of joint training, especially when the noise rate is high.

\begin{figure}[!h]
\vspace{-3mm}
	\centering
	\subfigure[Optimal $\theta$ for various noise rates on CIFAR-10.]{
		\includegraphics[width=0.46\linewidth]{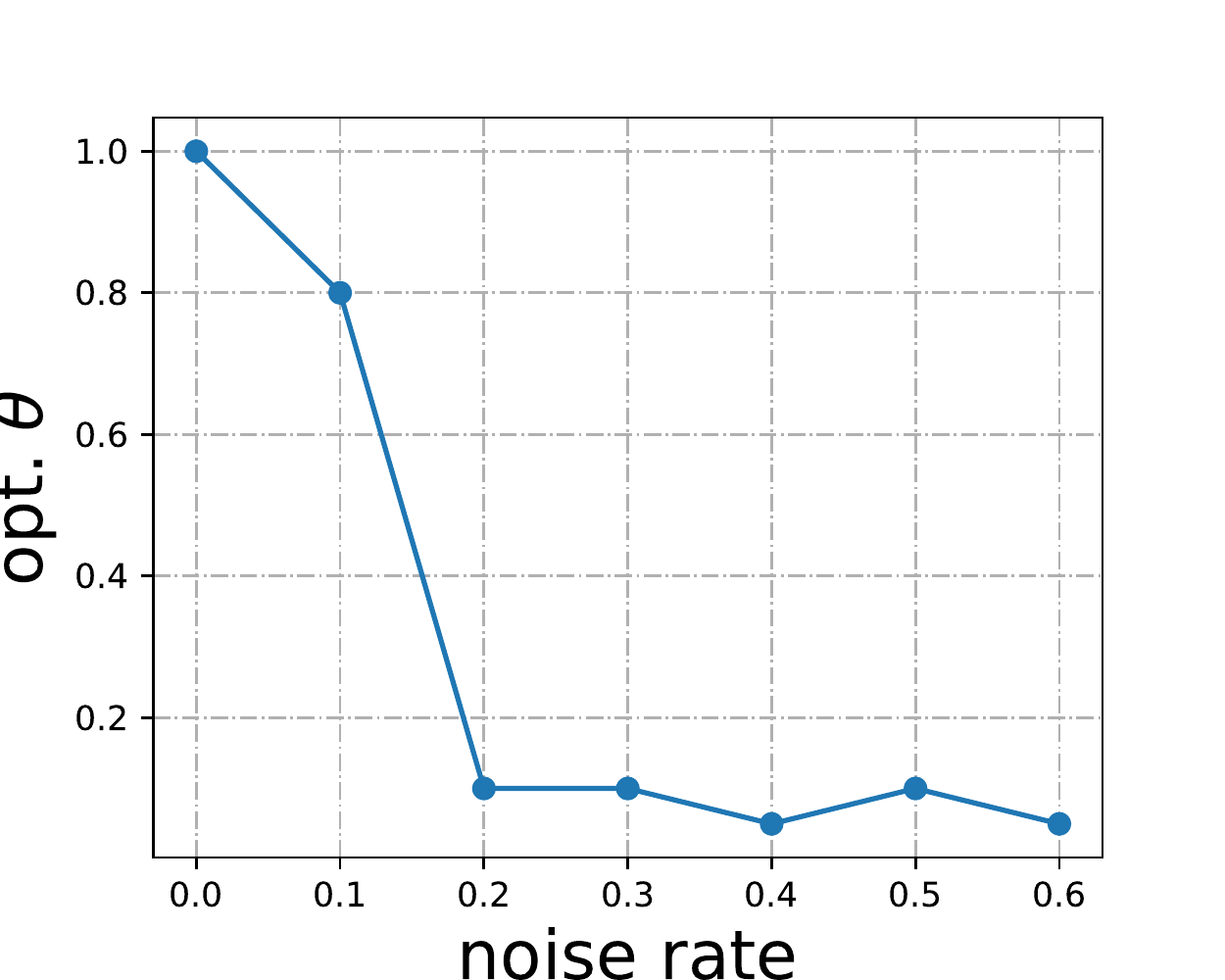}
		\label{fig::theta_cifar}
	}
\hfill
	\subfigure[Performance change w.r.t.~$\theta$ on TinyImageNet-200.]{
		\includegraphics[width=0.46\linewidth]{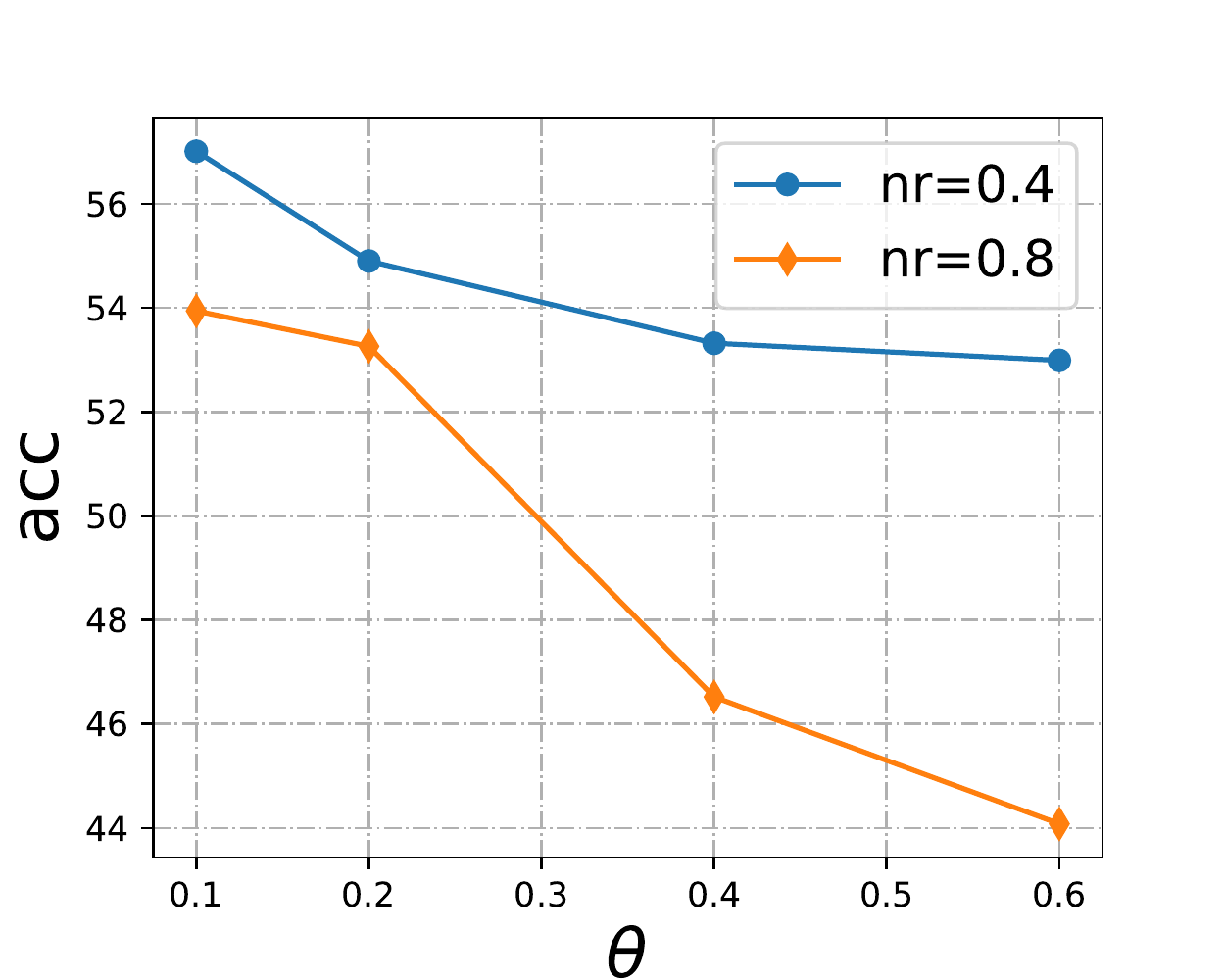}
		\label{fig::theta_img}
	}
         \vspace{-0.1 in} 
	\caption{Parameter analysis of $\theta$.}
	\label{fig::theta}
\end{figure}

\subsection{Parameter Analysis of $k$}

We empirically verify the discussions following Theorem \ref{thm::sample_bound_noisy} about the feature dimension $k$.
We train the joint objective \eqref{eq::graphmix} with $\theta=0.1$ using ResNet-50 on the CIFAR-10 dataset under $30\%$ label noise with feature dimension varying in $\{128,256,512,1024,2048\}$, and report the linear probing accuracy in Table \ref{tab::k}.

	\begin{table}[!h]
	\centering
	\caption{Parameter analysis of $k$.}
 \vspace{-0.1 in}
	\begin{tabular}{cccccc}
		\toprule
		$k$ & 128 & 256 & 512 & 1024 & 2048 \\
		\midrule
		Acc (\%) & 91.89 & \textbf{92.18} & 92.12 & 91.77 & 90.9 \\
		\bottomrule
	\end{tabular}
	\label{tab::k}
\end{table}

We observe that as $k$ increases, the linear probing accuracy first increases and then decreases, which validates the theoretical insight that we should choose a moderate feature dimension $k$ in Theorem \ref{thm::sample_bound_noisy}.

\subsection{The Winner of SupCon and SimCLR Serves as a Strong Baseline}\label{sec::exp_baseline}

\begin{figure}[!h]
	\centering
	\includegraphics[width=\linewidth]{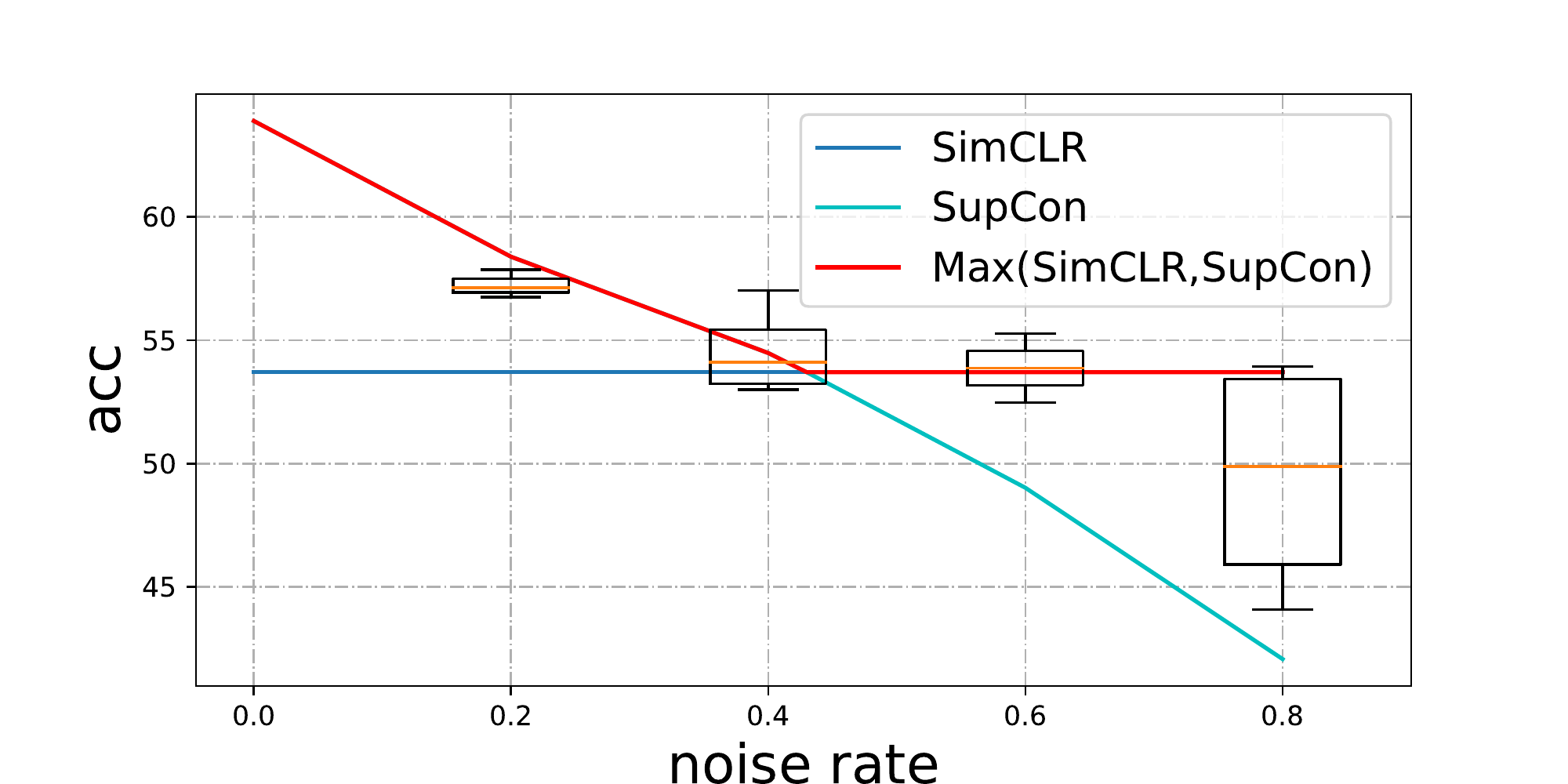}
	\caption{The winner of SupCon and SimCLR serves as a strong baseline of joint training under label noise on the TinyImageNet-200 dataset.}
	\label{fig::max_img}
  \vskip -0.2in
\end{figure}

In this section, we show that under label noise, joint training of SupCon and SimCLR has limited improvement over a simple but strong baseline, that is, the winner of SupCon and SimCLR on the TinyImageNet-200 dataset.
Specifically, in Figure \ref{fig::max_img}, the box plot show the linear probing performances of $\mathcal{L}_{\mathrm{JointTraining}}$ with $\theta \in \{0.1,0.2,0.4,0.6\}$ under noise rates $\{0.2,0.4,0.6,0.8\}$ respectively. And the curves respectively show the performance of SimCLR, SupCon, and the winner of the two. We can see that the ``Max'' performance of SupCon and SimCLR lies within the box plots, indicating that the winner of SupCon and SimCLR has a comparative or sometimes better performance compared with the jointly trained model. Additionally, we observe the same trends on tasks like detection, segmentation, and fine-tuning in Appendix \ref{app::transfer}. These experimental results verify the theoretical result in Theorem \ref{thm::pop_bound_noisy} that joint training does not improve the error bound.

\section{Conclusion}
In this paper, we establish a theoretical framework for weakly supervised contrastive learning, which is compatible with the settings of both noisy label learning and semi-supervised learning.
We take spectral contrastive learning as a proxy for theoretical analysis.
By formulating a mixed similarity graph induced by both weakly supervised label information and unsupervised feature information, we analyze the weakly supervised spectral contrastive learning based on the framework of spectral clustering, and derive the downstream linear evaluation error bound.
Our theoretical results show that semi-supervised information improves the downstream error bound, whereas, under the setting of symmetric label noise, we prove that jointly training supervised and unsupervised contrastive losses fail to improve the error bound. 
Our theoretical findings here provide new insights for the community to rethink the role of weak supervision in helping the representation of contrastive learning. 
For future works, we will investigate the effect of more complex weak supervision, such as active learning and label-dependent label noise, on contrastive learning. 

\section*{Acknowledgements}
Yisen Wang is partially supported by the National Key R\&D Program of China (2022ZD0160304), the National Natural Science Foundation of China (62006153), Open Research Projects of Zhejiang Lab (No. 2022RC0AB05), and Huawei Technologies Inc. Weiran Huang is funded by MSRA.

\bibliography{NLCL}

\begin{thebibliography}{40}
\providecommand{\natexlab}[1]{#1}
\providecommand{\url}[1]{\texttt{#1}}
\expandafter\ifx\csname urlstyle\endcsname\relax
  \providecommand{\doi}[1]{doi: #1}\else
  \providecommand{\doi}{doi: \begingroup \urlstyle{rm}\Url}\fi

\bibitem[Acharya et~al.(2022)Acharya, Sanghavi, Jing, Bhushanam, Choudhary,
  Rabbat, and Dhillon]{acharya2022positive}
Acharya, A., Sanghavi, S., Jing, L., Bhushanam, B., Choudhary, D., Rabbat, M.,
  and Dhillon, I.
\newblock Positive unlabeled contrastive learning.
\newblock \emph{arXiv preprint arXiv:2206.01206}, 2022.

\bibitem[Aitchison(2021)]{aitchison2021infonce}
Aitchison, L.
\newblock Infonce is a variational autoencoder.
\newblock \emph{arXiv preprint arXiv:2107.02495}, 2021.

\bibitem[Arora et~al.(2019)Arora, Khandeparkar, Khodak, Plevrakis, and
  Saunshi]{arora2019theoretical}
Arora, S., Khandeparkar, H., Khodak, M., Plevrakis, O., and Saunshi, N.
\newblock A theoretical analysis of contrastive unsupervised representation
  learning.
\newblock In \emph{ICML}, 2019.

\bibitem[Ash et~al.(2022)Ash, Goel, Krishnamurthy, and
  Misra]{ash2022investigating}
Ash, J., Goel, S., Krishnamurthy, A., and Misra, D.
\newblock Investigating the role of negatives in contrastive representation
  learning.
\newblock In \emph{AISTATS}, 2022.

\bibitem[Assran et~al.(2020)Assran, Ballas, Castrejon, and
  Rabbat]{assran2020supervision}
Assran, M., Ballas, N., Castrejon, L., and Rabbat, M.
\newblock Supervision accelerates pre-training in contrastive semi-supervised
  learning of visual representations.
\newblock \emph{arXiv preprint arXiv:2006.10803}, 2020.

\bibitem[Bao et~al.(2022)Bao, Nagano, and Nozawa]{bao2022surrogate}
Bao, H., Nagano, Y., and Nozawa, K.
\newblock On the surrogate gap between contrastive and supervised losses.
\newblock In \emph{ICML}, 2022.

\bibitem[Chen et~al.(2022)Chen, Fu, Narayan, Zhang, Song, Fatahalian, and
  R{\'e}]{chen2022perfectly}
Chen, M., Fu, D.~Y., Narayan, A., Zhang, M., Song, Z., Fatahalian, K., and
  R{\'e}, C.
\newblock Perfectly balanced: Improving transfer and robustness of supervised
  contrastive learning.
\newblock In \emph{ICML}, 2022.

\bibitem[Chen et~al.(2020)Chen, Kornblith, Norouzi, and Hinton]{chen2020simple}
Chen, T., Kornblith, S., Norouzi, M., and Hinton, G.
\newblock A simple framework for contrastive learning of visual
  representations.
\newblock In \emph{ICML}, 2020.

\bibitem[Chen \& He(2021)Chen and He]{chen2021exploring}
Chen, X. and He, K.
\newblock Exploring simple siamese representation learning.
\newblock In \emph{CVPR}, 2021.

\bibitem[Cheng et~al.(2021)Cheng, Zhu, Sun, and Liu]{cheng2021demystifying}
Cheng, H., Zhu, Z., Sun, X., and Liu, Y.
\newblock Demystifying how self-supervised features improve training from noisy
  labels.
\newblock \emph{arXiv preprint arXiv:2110.09022}, 2021.

\bibitem[Chuang et~al.(2022)Chuang, Hjelm, Wang, Vineet, Joshi, Torralba,
  Jegelka, and Song]{chuang2022robust}
Chuang, C.-Y., Hjelm, R.~D., Wang, X., Vineet, V., Joshi, N., Torralba, A.,
  Jegelka, S., and Song, Y.
\newblock Robust contrastive learning against noisy views.
\newblock In \emph{CVPR}, 2022.

\bibitem[Fulton(2000)]{fulton2000eigenvalues}
Fulton, W.
\newblock Eigenvalues, invariant factors, highest weights, and schubert
  calculus.
\newblock \emph{Bulletin of the American Mathematical Society}, 37\penalty0
  (3):\penalty0 209--249, 2000.

\bibitem[Ghosh \& Lan(2021)Ghosh and Lan]{ghosh2021contrastive}
Ghosh, A. and Lan, A.
\newblock Contrastive learning improves model robustness under label noise.
\newblock In \emph{CVPR}, 2021.

\bibitem[HaoChen et~al.(2021)HaoChen, Wei, Gaidon, and Ma]{haochen2021provable}
HaoChen, J.~Z., Wei, C., Gaidon, A., and Ma, T.
\newblock Provable guarantees for self-supervised deep learning with spectral
  contrastive loss.
\newblock In \emph{NeurIPS}, 2021.

\bibitem[He et~al.(2020)He, Fan, Wu, Xie, and Girshick]{he2020momentum}
He, K., Fan, H., Wu, Y., Xie, S., and Girshick, R.
\newblock Momentum contrast for unsupervised visual representation learning.
\newblock In \emph{CVPR}, 2020.

\bibitem[Hu et~al.(2023)Hu, Liu, Zhou, Wang, and Huang]{hu2022your}
Hu, T., Liu, Z., Zhou, F., Wang, W., and Huang, W.
\newblock Your contrastive learning is secretly doing stochastic neighbor
  embedding.
\newblock In \emph{ICLR}, 2023.

\bibitem[Huang et~al.(2023)Huang, Yi, Zhao, and Jiang]{huang2023towards}
Huang, W., Yi, M., Zhao, X., and Jiang, Z.
\newblock Towards the generalization of contrastive self-supervised learning.
\newblock In \emph{ICLR}, 2023.

\bibitem[Islam et~al.(2021)Islam, Chen, Panda, Karlinsky, Radke, and
  Feris]{islam2021broad}
Islam, A., Chen, C.-F.~R., Panda, R., Karlinsky, L., Radke, R., and Feris, R.
\newblock A broad study on the transferability of visual representations with
  contrastive learning.
\newblock In \emph{ICCV}, pp.\  8845--8855, 2021.

\bibitem[Johnson et~al.(2023)Johnson, Hanchi, and
  Maddison]{johnson2022contrastive}
Johnson, D.~D., Hanchi, A.~E., and Maddison, C.~J.
\newblock Contrastive learning can find an optimal basis for approximately
  view-invariant functions.
\newblock In \emph{ICLR}, 2023.

\bibitem[Khosla et~al.(2020)Khosla, Teterwak, Wang, Sarna, Tian, Isola,
  Maschinot, Liu, and Krishnan]{khosla2020supervised}
Khosla, P., Teterwak, P., Wang, C., Sarna, A., Tian, Y., Isola, P., Maschinot,
  A., Liu, C., and Krishnan, D.
\newblock Supervised contrastive learning.
\newblock In \emph{NeurIPS}, 2020.

\bibitem[Kim et~al.(2021{\natexlab{a}})Kim, Choo, Kwon, Joe, Min, and
  Gwon]{kim2021selfmatch}
Kim, B., Choo, J., Kwon, Y.-D., Joe, S., Min, S., and Gwon, Y.
\newblock Selfmatch: Combining contrastive self-supervision and consistency for
  semi-supervised learning.
\newblock \emph{arXiv preprint arXiv:2101.06480}, 2021{\natexlab{a}}.

\bibitem[Kim et~al.(2019)Kim, Yim, Yun, and Kim]{kim2019nlnl}
Kim, Y., Yim, J., Yun, J., and Kim, J.
\newblock Nlnl: Negative learning for noisy labels.
\newblock In \emph{ICCV}, 2019.

\bibitem[Kim et~al.(2021{\natexlab{b}})Kim, Yun, Shon, and Kim]{kim2021joint}
Kim, Y., Yun, J., Shon, H., and Kim, J.
\newblock Joint negative and positive learning for noisy labels.
\newblock In \emph{CVPR}, 2021{\natexlab{b}}.

\bibitem[Ko et~al.(2022)Ko, Mohapatra, Liu, Chen, Daniel, and
  Weng]{ko2022revisiting}
Ko, C.-Y., Mohapatra, J., Liu, S., Chen, P.-Y., Daniel, L., and Weng, L.
\newblock Revisiting contrastive learning through the lens of neighborhood
  component analysis: an integrated framework.
\newblock In \emph{ICML}, 2022.

\bibitem[Lee et~al.(2022)Lee, Kim, Kim, Cheon, Cho, and
  Han]{lee2022contrastive}
Lee, D., Kim, S., Kim, I., Cheon, Y., Cho, M., and Han, W.-S.
\newblock Contrastive regularization for semi-supervised learning.
\newblock In \emph{CVPR}, 2022.

\bibitem[Li et~al.(2022)Li, Xia, Ge, and Liu]{li2022selective}
Li, S., Xia, X., Ge, S., and Liu, T.
\newblock Selective-supervised contrastive learning with noisy labels.
\newblock In \emph{CVPR}, 2022.

\bibitem[Navaneet et~al.(2022)Navaneet, Abbasi~Koohpayegani, Tejankar,
  Pourahmadi, Subramanya, and Pirsiavash]{navaneet2022constrained}
Navaneet, K., Abbasi~Koohpayegani, S., Tejankar, A., Pourahmadi, K.,
  Subramanya, A., and Pirsiavash, H.
\newblock Constrained mean shift using distant yet related neighbors for
  representation learning.
\newblock In \emph{ECCV 2022}, 2022.

\bibitem[Nozawa \& Sato(2021)Nozawa and Sato]{nozawa2021understanding}
Nozawa, K. and Sato, I.
\newblock Understanding negative samples in instance discriminative
  self-supervised representation learning.
\newblock In \emph{NeurIPS}, 2021.

\bibitem[Ortego et~al.(2021)Ortego, Arazo, Albert, O'Connor, and
  McGuinness]{ortego2021multi}
Ortego, D., Arazo, E., Albert, P., O'Connor, N.~E., and McGuinness, K.
\newblock Multi-objective interpolation training for robustness to label noise.
\newblock In \emph{CVPR}, 2021.

\bibitem[Shen et~al.(2022)Shen, Jones, Kumar, Xie, HaoChen, Ma, and
  Liang]{shen2022connect}
Shen, K., Jones, R.~M., Kumar, A., Xie, S.~M., HaoChen, J.~Z., Ma, T., and
  Liang, P.
\newblock Connect, not collapse: Explaining contrastive learning for
  unsupervised domain adaptation.
\newblock In \emph{ICML}, 2022.

\bibitem[Wang et~al.(2021{\natexlab{a}})Wang, Geng, Jiang, Li, Wang, Yang, and
  Lin]{wang2021residual}
Wang, Y., Geng, Z., Jiang, F., Li, C., Wang, Y., Yang, J., and Lin, Z.
\newblock Residual relaxation for multi-view representation learning.
\newblock In \emph{NeurIPS}, 2021{\natexlab{a}}.

\bibitem[Wang et~al.(2021{\natexlab{b}})Wang, Zhang, Wang, Yang, and
  Lin]{wang2021chaos}
Wang, Y., Zhang, Q., Wang, Y., Yang, J., and Lin, Z.
\newblock Chaos is a ladder: A new theoretical understanding of contrastive
  learning via augmentation overlap.
\newblock In \emph{ICLR}, 2021{\natexlab{b}}.

\bibitem[Wang et~al.(2023)Wang, Zhang, Du, Yang, Lin, and
  Wang]{wang2023message}
Wang, Y., Zhang, Q., Du, T., Yang, J., Lin, Z., and Wang, Y.
\newblock A message passing perspective on learning dynamics of contrastive
  learning.
\newblock In \emph{ICLR}, 2023.

\bibitem[Xue et~al.(2022)Xue, Whitecross, and
  Mirzasoleiman]{xue2022investigating}
Xue, Y., Whitecross, K., and Mirzasoleiman, B.
\newblock Investigating why contrastive learning benefits robustness against
  label noise.
\newblock In \emph{ICML}, 2022.

\bibitem[Yan et~al.(2022)Yan, Luo, Xu, Deng, and Huang]{yan2022noise}
Yan, J., Luo, L., Xu, C., Deng, C., and Huang, H.
\newblock Noise is also useful: Negative correlation-steered latent contrastive
  learning.
\newblock In \emph{CVPR}, 2022.

\bibitem[Yang et~al.(2022)Yang, Wu, Zhang, Jiang, Liu, Zheng, Zhang, Wang, and
  Zeng]{yang2022class}
Yang, F., Wu, K., Zhang, S., Jiang, G., Liu, Y., Zheng, F., Zhang, W., Wang,
  C., and Zeng, L.
\newblock Class-aware contrastive semi-supervised learning.
\newblock In \emph{CVPR}, 2022.

\bibitem[Yao et~al.(2021)Yao, Sun, Zhang, Shen, Wu, Zhang, and Tang]{yao2021jo}
Yao, Y., Sun, Z., Zhang, C., Shen, F., Wu, Q., Zhang, J., and Tang, Z.
\newblock Jo-src: A contrastive approach for combating noisy labels.
\newblock In \emph{CVPR}, 2021.

\bibitem[Zhang et~al.(2022{\natexlab{a}})Zhang, Yuan, Yao, and
  Huang]{zhang2022learning}
Zhang, M., Yuan, C., Yao, J., and Huang, W.
\newblock Learning with noisily-labeled class-imbalanced data.
\newblock \emph{arXiv preprint arXiv:2211.10955}, 2022{\natexlab{a}}.

\bibitem[Zhang et~al.(2022{\natexlab{b}})Zhang, Zhang, Li, Qiu, Xu, and
  Tian]{zhang2022semi}
Zhang, Y., Zhang, X., Li, J., Qiu, R., Xu, H., and Tian, Q.
\newblock Semi-supervised contrastive learning with similarity co-calibration.
\newblock \emph{IEEE Transactions on Multimedia}, 2022{\natexlab{b}}.

\bibitem[Zimmermann et~al.(2021)Zimmermann, Sharma, Schneider, Bethge, and
  Brendel]{zimmermann2021contrastive}
Zimmermann, R.~S., Sharma, Y., Schneider, S., Bethge, M., and Brendel, W.
\newblock Contrastive learning inverts the data generating process.
\newblock In \emph{ICML}, 2021.

\end{thebibliography}
\bibliographystyle{ICML2023}

\clearpage
\appendix
\onecolumn
\begin{center}
    \Large \bf Appendix
\end{center}

The appendix consists of the proofs for lemmas and theorems (Section~\ref{sec::proof}) and additional experimental results (Section~\ref{sec::add_exp}).

\section{Proofs}\label{sec::proof}

\subsection{Proof of Lemma \ref{lem::relation}}

\begin{proof}
	Under Assumption \ref{ass::symnoise}, we have
	\begin{align}
		(\boldsymbol{T}^2)_{i,j}
		&=\left\{
		\begin{aligned}
			&(1-\gamma)^2 + \gamma^2/(r-1), &\ i=j \\
			&2\gamma(1-\gamma)/(r-1) + (r-2)\gamma^2/(r-1)^2, &\ i\neq j \\
		\end{aligned}
		\right.
		\nonumber\\
		&=\left\{
		\begin{aligned}
			& (1-\gamma)^2 + \gamma^2/(r-1), &\ i=j \\
			& \frac{\gamma}{r-1}\Big(2-\frac{r}{r-1}\gamma\Big), &\ i\neq j. \\
		\end{aligned}
		\right.
	\end{align}
	That is, we have 
	\begin{align}
		\boldsymbol{T}^2 
		&= \Big[(1-\gamma)^2 + \gamma^2/(r-1)-\frac{\gamma}{r-1}\Big(2-\frac{r}{r-1}\gamma\Big)\Big] \boldsymbol{I}_{r \times r} 
		\nonumber\\
		&\qquad + \frac{\gamma}{r-1}\Big(2-\frac{r}{r-1}\gamma\Big) \vec{1}_r\vec{1}_r^\top 
		\nonumber\\
		&= \Big(1-\frac{r}{r-1}\gamma\Big)^2 \boldsymbol{I}_{r \times r} + \frac{\gamma}{r-1}\Big(2-\frac{r}{r-1}\gamma\Big) \vec{1}_r\vec{1}_r^\top
		\nonumber\\
		&:= \alpha \boldsymbol{I}_{r \times r} + \beta \vec{1}_r\vec{1}_r^\top.
	\end{align}
	Given $\gamma \in [0,1)$, we have
	\begin{align}
		\boldsymbol{A}^{\star}_L 
		= \boldsymbol{Y}_L \boldsymbol{T}^2 \boldsymbol{Y}_L^\top 
		&= \boldsymbol{Y}_L \big(\alpha \boldsymbol{I}_{r \times r} + \beta \vec{1}_r\vec{1}_r^\top \big) \boldsymbol{Y}_L^\top 
		\nonumber\\
		&= \alpha \boldsymbol{Y}_L \boldsymbol{Y}_L^\top + \beta \boldsymbol{Y}_L \vec{1}_r\vec{1}_r^\top \boldsymbol{Y}_L^\top 
		\nonumber\\
		&= \alpha \boldsymbol{A}_L + \beta \vec{1}_{n_L}\vec{1}_{n_L}^\top,
	\end{align}
	where the last equality holds because $\sum_{j} \eta_j(x_i)=1$ for $i \in [n]$.
\end{proof}

and the normalized augmentation graph is 
\begin{align}
	\bar{\boldsymbol{A}}^{\star}
	= \tilde{\boldsymbol{D}}^{-1/2} \boldsymbol{A}^{\star} \tilde{\boldsymbol{D}}^{-1/2},
\end{align}
where 
\begin{align}
	\tilde{\boldsymbol{D}} = 
	\begin{bmatrix}
		\tilde{\boldsymbol{D}}_L & \boldsymbol{0} \\
		\boldsymbol{0} & \boldsymbol{I}_{n_U \times n_U}
	\end{bmatrix},
\end{align}
\begin{align}
	\tilde{\boldsymbol{D}}_L = \mathrm{diag}(d_i),
\end{align}
and
\begin{align}
	d_i &= \sum_{j \in [n_L]} \boldsymbol{A}^{\star}_{i,j} = \alpha \sum_{j \in [n_L]} \sum_{\ell \in [r]} \eta_\ell(x_i) \eta_\ell(x_j) + n_L\beta
	\nonumber\\
	&= \alpha \sum_{\ell \in [r]} \eta_\ell(x_i) \sum_{j \in [n_L]} \eta_\ell(x_j) + n_L\beta
	= \alpha \sum_{\ell \in [r]} \eta_\ell(x_i) n_\ell + n_L\beta
\end{align}

Specifically, when the labeled data is class-balanced, i.e. $n_1=\ldots=n_r=n_L/r$. Then we have
\begin{align}
	d_i = \frac{n_L}{r} \alpha \sum_{\ell \in [r]} \eta_\ell(x_i) + n_L\beta
	= \frac{n_L}{r}\alpha + n_L\beta = \frac{n_L}{r},
\end{align}
and thus 
\begin{align}
	\bar{\boldsymbol{A}}^{\star} = 
	\begin{bmatrix}
		\alpha \frac{r}{n_L} \boldsymbol{A}_L + \beta \frac{r}{n_L} \vec{1}_{n_L}\vec{1}_{n_L}^\top & \boldsymbol{0} \\
		\boldsymbol{0} & \boldsymbol{I}_{n_U \times n_U}
	\end{bmatrix}.
\end{align}

\subsection{Proof of Proposition \ref{lem::tilde_lambda}}

\begin{proof}
	We first prove that $v_1 = \frac{1}{\sqrt{n_L}}\vec{1}_{n_L}$ is an eigenvector of $\bar{\boldsymbol{A}}_L:=\frac{r}{n_L} \boldsymbol{A}_L$ with eigenvalue $\mu_1 = 1$. To be specific, 
	\begin{align}
		\bar{\boldsymbol{A}}_L \cdot \frac{1}{\sqrt{n_L}}\vec{1}_{n_L}
		&= \frac{1}{\sqrt{n_L}} \cdot 
		\frac{r}{n_L} \boldsymbol{A}_L 
		\cdot \vec{1}_{n_L}
		\nonumber\\
		&= \frac{1}{\sqrt{n_L}} \cdot 
		\frac{r}{n_L} \boldsymbol{Y}_L\boldsymbol{Y}_L^\top \vec{1}_{n_L} 
		\nonumber\\
		&= \frac{1}{\sqrt{n_L}} \cdot
		\frac{r}{n_L} \boldsymbol{Y}_L \frac{n_L}{r} \vec{1}_r 
		\nonumber\\
		&= \frac{1}{\sqrt{n_L}} \vec{1}_{n_L},
	\end{align}
	where the second last equality is due to class balance, i.e. $\sum_{i \in [n_L]} \eta_j(x_i)=n_L/r$ for $j \in [r]$, and the last equality holds because $\sum_{j \in [r]} \eta_j(x_i) = 1$ for $i \in [n_L]$.
	
	Therefore, we can rewrite $\bar{\boldsymbol{A}}_L$ as 
	\begin{align}
		\bar{\boldsymbol{A}}_L
		= 
		\begin{bmatrix}
			\frac{1}{\sqrt{n_L}} \vec{1}_{n_L}, v_2, \ldots, v_{n_L}
		\end{bmatrix}
		\begin{bmatrix}
			1 & 0 & \ldots & 0 \\
			0 & \mu_2 & \ldots & 0 \\
			\vdots & \vdots & & \vdots \\
			0 & 0 & \ldots & \mu_{n_L}
		\end{bmatrix}
		\begin{bmatrix}
			\frac{1}{\sqrt{n_L}} \vec{1}_{n_L}^\top \\
			v_2^\top \\
			\vdots \\
			v_{n_L}^\top
		\end{bmatrix}.
	\end{align}
	
	Note that $\frac{1}{n_L}\vec{1}_{n_L}\vec{1}_{n_L}^\top$ can be decomposed as
	\begin{align}
		\frac{1}{n_L}\vec{1}_{n_L}\vec{1}_{n_L}^\top
		&= \Big(\frac{1}{\sqrt{n_L}}\vec{1}_{n_L}\Big) \Big(\frac{1}{\sqrt{n_L}}\vec{1}_{n_L}\Big)^\top
		\nonumber\\
		&= 
		\begin{bmatrix}
			\frac{1}{\sqrt{n_L}} \vec{1}_{n_L}, v_2, \ldots, v_{n_L}
		\end{bmatrix}
		\begin{bmatrix}
			1 & 0 & \ldots & 0 \\
			0 & 0 & \ldots & 0 \\
			\vdots & \vdots & & \vdots \\
			0 & 0 & \ldots & 0
		\end{bmatrix}
		\begin{bmatrix}
			\frac{1}{\sqrt{n_L}} \vec{1}_{n_L}^\top \\
			v_2^\top \\
			\vdots \\
			v_{n_L}^\top
		\end{bmatrix}.
	\end{align}
	Then we have 
	\begin{align}
		\bar{\boldsymbol{A}}^{\star}_L 
		&:= \alpha \frac{r}{n_L} \boldsymbol{A}_L + r\beta \frac{1}{n_L} \vec{1}_{n_L}\vec{1}_{n_L}^\top 
		\nonumber\\
		&= 
		\begin{bmatrix}
			\frac{1}{\sqrt{n_L}} \vec{1}_{n_L}, v_2, \ldots, v_{n_L}
		\end{bmatrix}
		\begin{bmatrix}
			\alpha & 0 & \ldots & 0 \\
			0 & \alpha\mu_2 & \ldots & 0 \\
			\vdots & \vdots & & \vdots \\
			0 & 0 & \ldots & \alpha\mu_{n_L}
		\end{bmatrix}
		\begin{bmatrix}
			\frac{1}{\sqrt{n_L}} \vec{1}_{n_L}^\top \\
			v_2^\top \\
			\vdots \\
			v_{n_L}^\top
		\end{bmatrix}
		\nonumber\\
		&\phantom{=}\cdot
		\begin{bmatrix}
			\frac{1}{\sqrt{n_L}} \vec{1}_{n_L}, v_2, \ldots, v_{n_L}
		\end{bmatrix}
		\begin{bmatrix}
			r\beta & 0 & \ldots & 0 \\
			0 & 0 & \ldots & 0 \\
			\vdots & \vdots & & \vdots \\
			0 & 0 & \ldots & 0
		\end{bmatrix}
		\begin{bmatrix}
			\frac{1}{\sqrt{n_L}} \vec{1}_{n_L}^\top \\
			v_2^\top \\
			\vdots \\
			v_{n_L}^\top
		\end{bmatrix}
		\nonumber\\
		&= 
		\begin{bmatrix}
			\frac{1}{\sqrt{n_L}} \vec{1}_{n_L}, v_2, \ldots, v_{n_L}
		\end{bmatrix}
		\begin{bmatrix}
			\alpha + r\beta & 0 & \ldots & 0 \\
			0 & \alpha\mu_2 & \ldots & 0 \\
			\vdots & \vdots & & \vdots \\
			0 & 0 & \ldots & \alpha\mu_{n_L}
		\end{bmatrix}
		\begin{bmatrix}
			\frac{1}{\sqrt{n_L}} \vec{1}_{n_L}^\top \\
			v_2^\top \\
			\vdots \\
			v_{n_L}^\top
		\end{bmatrix}.
	\end{align}
	Since $\alpha+r\beta=1$, the eigenvalues of $\bar{\boldsymbol{A}}^{\star}_L$ are $1, \alpha\mu_2, \ldots, \alpha\mu_{n_L}$.
	Thus the eigenvalues of 
	\begin{align}
		\bar{\boldsymbol{A}}^{\star} = 
		\begin{bmatrix}
			\bar{\boldsymbol{A}}^{\star}_L & \boldsymbol{0} \\
			\boldsymbol{0} & \boldsymbol{I}_{n_U \times n_U}
		\end{bmatrix}
	\end{align}
	are 
	\begin{align}
		& \tilde{\mu}_1 = \ldots = \tilde{\mu}_{n_U+1} = 1, \\
		& \tilde{\mu}_j = \alpha\mu_j, \text{ for } j = n_U+2, \ldots, n.
	\end{align}
\end{proof}

\subsection{Proof of Proposition \ref{lem::eigenmix}}

\begin{proof}
	By equation 13 in \citet{fulton2000eigenvalues}, for two real symmetric $n$ by $n$ matrix $(1-\theta)\bar{\boldsymbol{A}}_0$ and $\theta\bar{\boldsymbol{A}}^{\star}$, the $k+1$-th largest eigenvalue of $\boldsymbol{A}_{\theta,\lambda,n_L} := (1-\theta)\bar{\boldsymbol{A}}_0 + \theta\bar{\boldsymbol{A}}^{\star}$ can take any value in the interval
	\begin{align}
		\max_{i+j=n+k+1} (1-\theta)\nu_i + \theta\tilde{\mu}_j \leq \lambda_{k+1} \leq \min_{i+j=k+2} (1-\theta)\nu_i + \theta\tilde{\mu}_j.
	\end{align}
	By Proposition \ref{lem::tilde_lambda}, we have
	\begin{align}
		\tilde{\mu}_j = 
		\left\{
		\begin{aligned}
			&1, & j=1, \ldots, n_U+1; \\
			&\alpha\mu_j, & j=n_U+2, \ldots, n_U+r; \\
			&0, & j=n_U+r+1, \ldots, n.\\
		\end{aligned}
		\right.
	\end{align}
	Therefore, we have
	\begin{align}
		&\phantom{=}\max_{i+j=n+k+1} (1-\theta)\nu_i + \theta\tilde{\mu}_j 
		\nonumber\\
		&= \max_{1 \leq i \leq n+k+1} (1-\theta)\nu_i + \theta\tilde{\mu}_{n+k+1-i} 
		\nonumber\\
		&= \max
		\left\{
		\begin{aligned}
			&\theta+(1-\theta)\nu_i, &i=n_L+k, \ldots, n;\\
			&\theta\alpha\mu_{n+k+1-i}+(1-\theta)\nu_i, & i=n_L+k-r+1, \ldots, n_L+k-1;\\
			&(1-\theta)\nu_i, & i=k+1, \ldots, n_L+k-r
		\end{aligned}
		\right.
		\nonumber\\
		&= \max\left\{
		\begin{aligned}
			&\theta+(1-\theta)\nu_{n_L+k} \\
			&\theta\alpha\mu_{n+k+1-i}+(1-\theta)\nu_i, & i=n_L+k-r+1, \ldots, n_L+k-1;\\
			&(1-\theta)\nu_{k+1},
		\end{aligned}
		\right.
	\end{align}
	where the last equality holds because $\{\nu_i\}_{i\in[n]}$ is ranked in descending order. 
	Then when $k \leq n_U$,
	\begin{align}
		\lambda_{k+1} \geq \max\big\{\theta+(1-\theta)\nu_{n_L+k}, \max_{i=n_L+k-r+1, \ldots, n_L+k-1}\{\theta\alpha\mu_{n+k+1-i}+(1-\theta)\nu_i\}, (1-\theta)\nu_{k+1}\big\},
	\end{align}
	when $n_U < k < n_U+r$,
	\begin{align}
		\lambda_{k+1} \geq \max\big\{(1-\theta)\nu_{k+1}, \max_{i=n_L+k-r+1, \ldots, n_L+k-1}\{\theta\alpha\mu_{n+k+1-i}+(1-\theta)\nu_i\}\big\},
	\end{align}
	and when $k \geq n_U+r$,
	\begin{align}
		\lambda_{k+1} \geq (1-\theta)\nu_{k+1}.
	\end{align}
	
	On the other hand, we have
	\begin{align}
		&\phantom{=}\min_{i+j=k+2} (1-\theta)\nu_i + \theta\tilde{\mu}_j 
		\nonumber\\
		&= \min_{1 \leq i \leq k+2} (1-\theta)\nu_i + \theta\tilde{\mu}_{k+2-i}
		\nonumber\\
		&= \min \left\{
		\begin{aligned}
			&\theta + (1-\theta)\nu_i, &i = k+1-n_U, \ldots, k+1; \\
			&\theta\alpha\mu_{k+2-i} + (1-\theta)\nu_i, &i = k+2-r-n_U, \ldots, k-n_U; \\
			&(1-\theta)\nu_i, &i = k+2-n, \ldots, k+1-r-n_U.
		\end{aligned}
		\right.
		\nonumber\\
		&= \min \left\{
		\begin{aligned}
			&\theta + (1-\theta)\nu_{k+1}; \\
			&\theta\alpha\mu_{k+2-i} + (1-\theta)\nu_i, &i = k+2-r-n_U, \ldots, k-n_U; \\
			&(1-\theta)\nu_{k+1-r-n_U}.
		\end{aligned}
		\right.
	\end{align}
	Then when $k \leq n_U$, there holds
	\begin{align}
		\lambda_{k+1} \leq \theta + (1-\theta)\nu_{k+1},
	\end{align}
	when $n_U < k < n_U+r$, there holds
	\begin{align}
		\lambda_{k+1} \leq \min\Big\{\theta + (1-\theta)\nu_{k+1}, \min_{i=k+2-r-n_U, \ldots, k-n_U} \{\theta\alpha\mu_{k+2-i} + (1-\theta)\nu_i\}\Big\},
	\end{align}
	and when $k \geq n_U+r$, there holds
	\begin{align}
		\lambda_{k+1} \leq \min\Big\{\theta + (1-\theta)\nu_{k+1}, \min_{i=k+2-r-n_U, \ldots, k-n_U} \{\theta\alpha\mu_{k+2-i} + (1-\theta)\nu_i\}, (1-\theta)\nu_{k+1-r-n_U}\Big\}.
	\end{align}
	
\end{proof}

\subsection{Proof of Theorem \ref{thm::pop_bound_semi}}

To prove Theorem \ref{thm::pop_bound_semi}, we first prove Proposition \ref{thm::pop_bound_theta}.

\begin{proposition}\label{thm::pop_bound_theta}
	For arbitrary $\boldsymbol{Y}$, assume that the labeled data is class-balanced, i.e. $\sum_{i \in [n_L]} \eta_j(x_i)=n_L/r$ for $j \in [r]$. Denote $\nu_1, \ldots, \nu_n$ as the eigenvalues of $\bar{\boldsymbol{A}}_0$ (in descending order). 
	Denote $\mathcal{E} := \mathrm{P}_{\bar{x}\sim \mathcal{P}_{\bar{X}}, x \sim \mathcal{A}(\cdot|\bar{x})} 
	\big(g_{f^*_{\mathrm{pop}}, B^*} (x) \neq y(\bar{x}) \big)$ as the linear evaluation error, where $B^* \in \mathbb{R}^{r\times k}$ with norm $\|B^*\|_F \leq 1/\lambda_k$.
	Assume there exists $\rho > 0$, such that $w_i/w_j < \rho$, for $i, j \in [n]$.
	Then under the deterministic scenario and Assumptions \ref{ass::symnoise} and \ref{ass::superror}, for $k \leq n_U$,
	there holds
	\begin{align}
		\mathcal{E}	
		&\leq 
		\frac{2\big[2\delta_u + [\alpha(1+\rho)\delta_s - 2\delta_u + (1-\alpha)]\theta\big]}{1-\theta - (1-\theta)\nu_{k+1}} + 8\delta_u,
	\end{align}
	for $n_U+1 \leq k \leq n_U+r-1$, there holds
	\begin{align}
		\mathcal{E}
		&\leq 
		\frac{2\big[2\delta_u + [\alpha(1+\rho)\delta_s - 2\delta_u + (1-\alpha)]\theta\big]}{1-\min\{\theta + (1-\theta)\nu_{k+1}, \theta\alpha + (1-\theta)\nu_{k-n_U}\}} + 8\delta_u,
	\end{align}
	and for $k \geq n_U+r$, there holds
	\begin{align}
		\mathcal{E}
		\leq 
		\frac{2\big[2\delta_u + [\alpha(1+\rho)\delta_s - 2\delta_u + (1-\alpha)]\theta\big]}{1-\lambda(\nu; \theta, \alpha)} + 8\delta_u,
	\end{align}
	where 
	\begin{align}
		\lambda(\boldsymbol{\nu}; \theta, \alpha) &= \min\{\theta + (1-\theta)\nu_{k+1}, \theta\alpha + (1-\theta)\nu_{k-n_U},
		\nonumber\\
		& \qquad \quad \ (1-\theta)\nu_{k+1-r-n_U}\}
	\end{align}
	and $\alpha := \big(1-\frac{r}{r-1}\gamma\big)^2$.
\end{proposition}

\begin{proof}[Proof of Proposition \ref{thm::pop_bound_theta}]
	By Lemma B.3 of \citet{haochen2021provable}, for any labeling function $\vec{y}:\mathcal{X} \to [r]$, there exists a linear probe $B^* \in \mathbb{R}^{r\times k}$ with norm $\|B^*\|_F \leq 1/\lambda_k$ such that
	\begin{align}\label{eq::errorhaochen}
		\mathrm{P}_{\bar{x}\sim \mathcal{P}_{\bar{X}}, x \sim \mathcal{A}(\cdot|\bar{x})} 
		\Big(g_{f^*_{\mathrm{pop}}, B^*} (x) \neq y(\bar{x}) \Big)
		\leq \frac{2\phi^{\hat{y}}}{1-\lambda_{k+1}} + 8 \delta_u,
	\end{align}
	where according to the definition of $\boldsymbol{A}_{\theta,\lambda,n_L}$, 
	\begin{align}\label{eq::phiy}
		\phi^{\hat{y}} 
		&= \sum_{i,j \in [n]}\big[(1-\theta) \frac{w_{ij}}{\sqrt{w_i w_j}} + \theta \bar{\boldsymbol{A}}^{\star}_{i,j}\big] \boldsymbol{1}[\hat{y}(x_i) \neq \hat{y}(x_j)]
		\nonumber \\
		&= \frac{1}{\sqrt{w_i w_j}} \big[(1-\theta) \sum_{i,j \in [n]} w_{i,j}\boldsymbol{1}[\hat{y}(x_i) \neq \hat{y}(x_j)] + \theta \sum_{i,j \in [n]}\sqrt{w_i w_j} \bar{\boldsymbol{A}}^{\star}_{i,j} \boldsymbol{1}[\hat{y}(x_i) \neq \hat{y}(x_j)]\big].
	\end{align}
	We investigate the RHS of \eqref{eq::phiy} respectively. The first term is
	\begin{align}\label{eq::phi1}
		&\quad\sum_{i,j \in [n]} w_{i,j}\boldsymbol{1}[\hat{y}(x_i) \neq \hat{y}(x_j)]
		\nonumber\\
		&= \sum_{i,j \in [n]} \mathbb{E}_{\bar{x}\sim \mathcal{P}_{\bar{\mathcal{X}}}}\mathcal{A}(x_i|\bar{x}) \mathcal{A}(x_j|\bar{x})\boldsymbol{1}[\hat{y}(x_i) \neq \hat{y}(x_j)]
		\nonumber\\
		&\leq  \sum_{i,j \in [n]} \mathbb{E}_{\bar{x}\sim \mathcal{P}_{\bar{\mathcal{X}}}}\mathcal{A}(x_i|\bar{x}) \mathcal{A}(x_j|\bar{x})\big(\boldsymbol{1}[\hat{y}(x_i) \neq \hat{y}(\bar{x})] + \boldsymbol{1}[\hat{y}(x_j) \neq \hat{y}(\bar{x})]\big)
		\nonumber\\
		&= 2 \sum_{i \in [n]} \mathbb{E}_{\bar{x}\sim \mathcal{P}_{\bar{\mathcal{X}}}}\mathcal{A}(x_i|\bar{x}) \boldsymbol{1}[\hat{y}(x_i) \neq \hat{y}(\bar{x})]
		\nonumber\\
		&\leq 2\delta_u.
	\end{align}
	The second term is
	\begin{align}
		&\quad \sum_{i,j \in [n]} \sqrt{w_i w_j} \bar{\boldsymbol{A}}^{\star}_{i,j} \boldsymbol{1}[\hat{y}(x_i) \neq \hat{y}(x_j)]
		\nonumber\\
		&= \sum_{i,j \in [n_L]} \sqrt{w_i w_j} (\bar{\boldsymbol{A}}^{\star}_L)_{i,j} \boldsymbol{1}[\hat{y}(x_i) \neq \hat{y}(x_j)]
		+  \sum_{i>n_L} \sqrt{w_i w_j} \boldsymbol{1}[\hat{y}(x_i) \neq \hat{y}(x_i)]
		\nonumber\\
		&\qquad + 2 \sum_{i\leq n_L, j>n_L} \sqrt{w_i w_j} \bar{\boldsymbol{A}}^{\star}_{i,j} \boldsymbol{1}[\hat{y}(x_i) \neq \hat{y}(x_i)].
	\end{align}
	According to the definition of $\bar{\boldsymbol{A}}^{\star}$, the last two terms are equal to $0$. Then by Lemma \ref{lem::relation}, the second term on the RHS of \eqref{eq::phiy} becomes
	\begin{align}
		&\quad \sum_{i,j \in [n]} \sqrt{w_i w_j} \bar{\boldsymbol{A}}^{\star}_{i,j} \boldsymbol{1}[\hat{y}(x_i) \neq \hat{y}(x_j)]
		\nonumber\\
		&= \sum_{i,j \in [n_L]} \sqrt{w_i w_j} \Big(\alpha \bar{\boldsymbol{A}}_{i,j}+\beta\frac{r}{n_L}\Big) \boldsymbol{1}[\hat{y}(x_i) \neq \hat{y}(x_j)]
		\nonumber\\
		&\leq \alpha \frac{r}{n_L} \sum_{i,j\in[n_L]} \sum_{\ell \in [r]} \sqrt{w_i w_j} \eta_{\ell}(x_i)\eta_{\ell}(x_j) \Big(\boldsymbol{1}[\hat{y}(x_i) \neq \ell] + \boldsymbol{1}[\hat{y}(x_j) \neq \ell]\Big)
		\nonumber\\
		&\qquad 
		+ (1-\alpha) \frac{1}{n_L} \sum_{i,j\in[n_L]} \sqrt{w_i w_j} 
		\nonumber\\
		&\leq \alpha \frac{r}{n_L} \sum_{i,j\in[n_L]} \sum_{\ell \in [r]} \frac{1}{2}(w_i+w_j) \eta_{\ell}(x_i)\eta_{\ell}(x_j) \Big(\boldsymbol{1}[\hat{y}(x_i) \neq \ell] + \boldsymbol{1}[\hat{y}(x_j) \neq \ell]\Big)
		\nonumber\\
		&\qquad 
		+ (1-\alpha) \frac{1}{n_L} \sum_{i,j\in[n_L]} \frac{1}{2}(w_i+w_j)
		\nonumber\\
		&= \frac{1}{2} \alpha \frac{r}{n_L} \sum_{\ell \in [r]} \sum_{i,j\in[n_L]}
		w_i \eta_{\ell}(x_i)\eta_{\ell}(x_j) \boldsymbol{1}[\hat{y}(x_i) \neq \ell] 
		\nonumber\\
		&+ \frac{1}{2} \alpha \frac{r}{n_L} \sum_{\ell \in [r]} \sum_{i,j\in[n_L]} w_i \eta_{\ell}(x_i)\eta_{\ell}(x_j) \boldsymbol{1}[\hat{y}(x_j) \neq \ell] 
		\nonumber\\
		&+  \frac{1}{2} \alpha \frac{r}{n_L} \sum_{\ell \in [r]} \sum_{i,j\in[n_L]} w_j \eta_{\ell}(x_i)\eta_{\ell}(x_j) \boldsymbol{1}[\hat{y}(x_i) \neq \ell] 
		\nonumber\\
		&+  \frac{1}{2} \alpha \frac{r}{n_L} \sum_{\ell \in [r]} \sum_{i,j\in[n_L]}w_j \eta_{\ell}(x_i)\eta_{\ell}(x_j) \boldsymbol{1}[\hat{y}(x_j) \neq \ell]
		+ (1-\alpha) 
		\nonumber\\
		&\leq \frac{1}{2} \alpha \frac{r}{n_L} \sum_{\ell \in [r]} \sum_{i,j\in[n_L]}
		\rho w_j \eta_{\ell}(x_i)\eta_{\ell}(x_j) \boldsymbol{1}[\hat{y}(x_i) \neq \ell] 
		\nonumber\\
		&+ \frac{1}{2} \alpha \frac{r}{n_L} \sum_{\ell \in [r]} \sum_{i,j\in[n_L]} w_i \eta_{\ell}(x_i)\eta_{\ell}(x_j) \boldsymbol{1}[\hat{y}(x_j) \neq \ell] 
		\nonumber\\
		&+  \frac{1}{2} \alpha \frac{r}{n_L} \sum_{\ell \in [r]} \sum_{i,j\in[n_L]} w_j \eta_{\ell}(x_i)\eta_{\ell}(x_j) \boldsymbol{1}[\hat{y}(x_i) \neq \ell] 
		\nonumber\\
		&+  \frac{1}{2} \alpha \frac{r}{n_L} \sum_{\ell \in [r]} \sum_{i,j\in[n_L]} \rho w_i \eta_{\ell}(x_i)\eta_{\ell}(x_j) \boldsymbol{1}[\hat{y}(x_j) \neq \ell]
		+ (1-\alpha) 
		\nonumber\\
		&= \frac{1}{2} \alpha \frac{r}{n_L} \sum_{\ell \in [r]} \sum_{i\in[n_L]}
		\rho \pi_\ell \eta_{\ell}(x_i) \boldsymbol{1}[\hat{y}(x_i) \neq \ell] 
		\nonumber\\
		&+ \frac{1}{2} \alpha \frac{r}{n_L} \sum_{\ell \in [r]} \sum_{j\in[n_L]} \pi_\ell \eta_{\ell}(x_j) \boldsymbol{1}[\hat{y}(x_j) \neq \ell] 
		\nonumber\\
		&+  \frac{1}{2} \alpha \frac{r}{n_L} \sum_{\ell \in [r]} \sum_{i\in[n_L]} \pi_\ell \eta_{\ell}(x_i) \boldsymbol{1}[\hat{y}(x_i) \neq \ell] 
		\nonumber\\
		&+  \frac{1}{2} \alpha \frac{r}{n_L} \sum_{\ell \in [r]} \sum_{j\in[n_L]} \rho \pi_\ell \eta_{\ell}(x_j) \boldsymbol{1}[\hat{y}(x_j) \neq \ell]
		+ (1-\alpha) 
		\nonumber\\
		&= \alpha(1+\rho) \frac{1}{n_L} \sum_{\ell \in [r]} \sum_{i\in[n_L]}
		\eta_{\ell}(x_i) \boldsymbol{1}[\hat{y}(x_i) \neq \ell] + (1-\alpha) 
		\nonumber\\
		& \leq \alpha(1+\rho) \delta_s 
		+ (1-\alpha), \label{eq::phi2}
	\end{align}
	where we denote $\pi_\ell = \mathrm{P}(Y=\ell)$ and by class-balance, $\pi_\ell = \frac{1}{r}$.
	Then combining \eqref{eq::phiy}, \eqref{eq::phi1} and \eqref{eq::phi2}, we have 
	\begin{align}\label{eq::phiyhat}
		\phi^{\hat{y}} \leq 2 (1-\theta) \delta_u + \theta  \alpha(1+\rho) \delta_s
		+ \theta(1-\alpha)
		= 2\delta_u + [\alpha(1+\rho)\delta_s - 2\delta_u + (1-\alpha)]\theta.
	\end{align}
	
	Therefore, by \eqref{eq::errorhaochen}, we have 
	\begin{align}
		\mathcal{E}:=\mathrm{P}_{\bar{x}\sim \mathcal{P}_{\bar{X}}, x \sim \mathcal{A}(\cdot|\bar{x})} 
		\Big(g_{f^*_{\mathrm{pop}}, B^*} (x) \neq y(\bar{x}) \Big)
		\leq \frac{2\big[2\delta_u + [\alpha(1+\rho)\delta_s - 2\delta_u + (1-\alpha)]\theta\big]}{1-\lambda_{k+1}} + 8 \delta_u,
	\end{align}
	Combined with Proposition \ref{lem::eigenmix}, we have  for $k \leq n_U$,
	there holds
	\begin{align}
		\mathcal{E}	
		&\leq 
		\frac{2\big[2\delta_u + [\alpha(1+\rho)\delta_s - 2\delta_u + (1-\alpha)]\theta\big]}{1-\theta - (1-\theta)\nu_{k+1}} + 8\delta_u,
	\end{align}
	for $n_U+1 \leq k \leq n_U+r-1$, there holds
	\begin{align}
		\mathcal{E}
		&\leq 
		\frac{2\big[2\delta_u + [\alpha(1+\rho)\delta_s - 2\delta_u + (1-\alpha)]\theta\big]}{1-\min\{\theta + (1-\theta)\nu_{k+1}, \theta\alpha + (1-\theta)\nu_{k-n_U}\}} + 8\delta_u,
	\end{align}
	and for $k \geq n_U+r$, there holds
	\begin{align}
		\mathcal{E}
		\leq 
		\frac{2\big[2\delta_u + [\alpha(1+\rho)\delta_s - 2\delta_u + (1-\alpha)]\theta\big]}{1-\min\{\theta + (1-\theta)\nu_{k+1}, \theta\alpha + (1-\theta)\nu_{k-n_U}, (1-\theta)\nu_{k+1-r-n_U}\}} + 8\delta_u.
	\end{align}
\end{proof}

\begin{proof}
	By taking $\gamma=0$, i.e. $\alpha=1$, and $k \leq n_U$ in Proposition \ref{thm::pop_bound_theta}, we reach the bound in \eqref{eq::boundsemi}.
	
	Besides, by Proposition \ref{lem::eigenmix}, we have
	\begin{align*}
		\|B^*\|_F 
		&\leq 1/\lambda_k
		\\
		&\leq 1/\max\Big\{\theta+(1-\theta)\nu_{n_L+k-1},  (1-\theta)\nu_{k},
		\theta\alpha+(1-\theta)\nu_{n_L+k-r}\Big\}
		\\
		&\leq 1/\max\Big\{\theta+(1-\theta)\nu_{n_L+k-1},  (1-\theta)\nu_{k},
		\theta+(1-\theta)\nu_{n_L+k-r}\Big\}
		\\
		&\leq 1/\max\Big\{(1-\theta)\nu_{k},
		\theta+(1-\theta)\nu_{n_L+k-r}\Big\}.
	\end{align*}
\end{proof}

\subsection{Proof of Theorem \ref{thm::pop_bound_noisy}}

\begin{proof}
	By taking $n_U=0$ and $k>r$ in Proposition \ref{thm::pop_bound_theta}, we reach the bound in \eqref{eq::boundnoisy}.
	
	Besides, by Proposition \ref{lem::eigenmix}, we have
	\begin{align*}
		\|B^*\|_F 
		\leq 1/\lambda_k
		\leq 1/(1-\theta)\nu_{k}.
	\end{align*}
\end{proof}

\subsection{Proof of Theorem \ref{thm::sample_bound_noisy}}

To prove Theorem \ref{thm::sample_bound_noisy}, we first derive the following generalization bound for spectral contrastive pretraining under label noise inspired by the proofs of Theorem 4.1 in \citet{haochen2021provable}. We note that the form seems to be the same as that of Theorem 4.1 in \citet{haochen2021provable}, but the underlying distribution of positive samples is different since we include the label similarity into consideration.

\begin{proposition}\label{prop::generalization_noisy}
	For some $\kappa > 0$, assume $\|f\|_{\infty} \leq \kappa$ for all $f \in \mathcal{F}$ and $x \in \mathcal{X}$. Let $f^*_{pop} \in \mathcal{F}$ be a minimizer of the population loss $\mathcal{L}(f)$. Given a random dataset of size $n$, let $\hat{f}_{emp} \in \mathcal{F}$ be a minimizer of empirical loss $\widehat{\mathcal{L}}_n(f)$. Then, with probability at least $1 - \varepsilon$ over the randomness of data, we have
	\begin{align}
		\mathcal{L}(\hat{f}_{emp}) \leq \mathcal{L}(f^*_{pop}) + c_1 \cdot \widehat{R}_{n/2}(\mathcal{F}) + c_2 \cdot \bigg(\sqrt{\frac{\log2/\delta}{n}}+\varepsilon\bigg),
	\end{align}
	where constants $c_1 \lesssim k^2\kappa^2+k\kappa$ and $c_2 \lesssim k\kappa^2 + k^2\kappa^4$.
\end{proposition}

\begin{proof}
	Inspired by the proofs of Theorem D.8 in \citet{haochen2021provable}, for $\hat{f} \in \argmin_{f:\mathcal{X}\to \mathbb{R}^k}\widehat{\mathcal{L}}_n(f)$ such that $\mathcal{L}(\hat{f}) \leq \min_{f:\mathcal{X}\to \mathbb{R}^k} \mathcal{L}(f) + \epsilon_0$, there holds
	\begin{align}\label{eq::generalization_noisy}
		\widehat{\mathcal{E}} \leq \min_{1\leq k'\leq k}\bigg(\frac{2\phi^{\hat{y}}}{1-\lambda_{k'+1}} + \frac{4k'\epsilon_0}{\big(\lambda_{k'} - \lambda_{k+1}\big)^2} \bigg) + 8\delta_u.
	\end{align}
	According to Proposition \ref{lem::eigenmix}, we have for $n_U=0$ and $k\geq r$,
	\begin{align}\label{eq::lambda_bound}
		(1-\theta)\nu_{k+1} \leq \lambda_{k+1}
		\leq \min\{\theta + (1-\theta)\nu_{k+1}, \theta\alpha + (1-\theta)\nu_{k}, 
		(1-\theta)\nu_{k+1-r}\}. 
	\end{align}
	By $k' \leq k$, there holds $\lambda_{k'} > \lambda_{k+1}$, and thus we have
	\begin{align}\label{eq::lambda_minus}
		\lambda_{k'} - \lambda_{k+1} \geq (1-\theta)\nu_{k'} - \min\{\theta + (1-\theta)\nu_{k+1}, \theta\alpha + (1-\theta)\nu_{k}, 
		(1-\theta)\nu_{k+1-r}\}\big.
	\end{align}
	Here we without loss of generality assume $1 \leq k' \leq k+1-r$ to ensure that \eqref{eq::lambda_minus} $\geq 0$.
	Then pluging \eqref{eq::phiyhat}, \eqref{eq::lambda_bound}, and \eqref{eq::lambda_minus} into \eqref{eq::generalization_noisy}, we finish the proof.
\end{proof}

\section{Additional Experiments}\label{sec::add_exp}

\subsection{Algorithms for Joint Training}\label{app::alg}

\textbf{Joint Training for Semi-supervised Learning.}
For semi-supervised contrastive learning, both labeled and unlabeled data are inputs of $L_{\mathrm{SimCLR}}$, that is, the unsupervised contrastive loss $L_{\mathrm{SimCLR}}$ is calculated based on all samples (denoted as $\{x_i\}_{i=1}^{n}$). The supervised contrastive loss $L_{\mathrm{SupCon}}$ is calculated based only on the labeled samples (without loss of generality, we denote the labeled samples as $\{(x_i,y_i)\}_{i=1}^{n_L}$). We show the procedures for semi-supervised contrastive learning in Algorithm \ref{alg::semi}.

\begin{algorithm}[!h]
	\caption{Joint Training for Semi-supervised Learning}
	\label{alg::semi}
	\begin{algorithmic}
		\STATE {\bfseries Input:} 
		Labeled data $\{(x_i, y_i)\}_{i=1}^{n_L}$; unlabeled data $\{x_i\}_{i=n_L+1}^{n}$; parameter $\theta$. \\
		\STATE {\bfseries Initialize:} Encoder $f$. \\
		\REPEAT
		\STATE Compute SupCon loss on the labeled data without using labels $\{(x_i, y_i)\}_{i=1}^{n_L}$;\\
		\STATE Compute SimCLR loss on all data $\{x_i\}_{i=1}^{n}$; \\
		\STATE Update encoder $f$ according to the joint training loss $(1-\theta)\frac{1}{n_L}\sum_{i=1}^{n_L} \mathcal{L}_{\mathrm{SupCon}}(f(x_i),y_i) + \theta \frac{1}{n}\sum_{i=1}^{n} \mathcal{L}_{\mathrm{SimCLR}}(f(x_i))$; \\
		\UNTIL Converge. \\
		\STATE {\bfseries Output:} Encoder $f$. \\
	\end{algorithmic}
\end{algorithm}

Note that in Section \ref{sec::graphsemi}, it seems that the unlabeled data is used in computing the SupCon loss, but this is not true. We include the unlabeled samples in the label similarity graph only for the convinience of mathematical formulations. In fact, as we view the unlabeled samples as having unique class labels, this formulation does not affect the selection of positive pairs and therefore does not affect the calculation of the SupCon loss.

\textbf{Joint Training for Noisy Label Learning.}
For joint training contrastive learning with label noise, the $L_{\mathrm{SupCon}}$ loss is calculated based on the noisy-labeled samples ${(x_i,\tilde{y}_i)}_{i=1}^n$, and the $L_{\mathrm{SimCLR}}$ loss is calculated based on all samples without using their labels. We show the procedures for joint training contrastive learning under label noise in Algorithm \ref{alg::noisy}.

\begin{algorithm}[!h]
	\caption{Joint Training for Noisy Label Learning}
	\label{alg::noisy}
	\begin{algorithmic}
		\STATE {\bfseries Input:} 
		Noisy labeled data $\{(x_i, \tilde{y}_i)\}_{i=1}^{n}$; parameter $\theta$. \\
		\STATE {\bfseries Initialize:} Encoder $f$. \\
		\REPEAT
		\STATE Compute SupCon loss on the (noisy) labeled data $\{(x_i, \tilde{y}_i)\}_{i=1}^{n}$;\\
		\STATE Compute SimCLR loss on all data without using labels $\{x_i\}_{i=1}^{n}$; \\
		\STATE Update encoder $f$ according to the joint training loss $(1-\theta)\frac{1}{n}\sum_{i=1}^{n} \mathcal{L}_{\mathrm{SupCon}}(f(x_i),\tilde{y}_i) + \theta \frac{1}{n}\sum_{i=1}^{n} \mathcal{L}_{\mathrm{SimCLR}}(f(x_i))$; \\
		\UNTIL Converge. \\
		\STATE {\bfseries Output:} Encoder $f$. \\
	\end{algorithmic}
\end{algorithm}

\subsection{Algorithmic Comparisons with Similar Methods}\label{app::compare}

\citet{islam2021broad} and \citet{chen2022perfectly} adopt similar methods as ours, but focus on different aspects. Specifically, \citet{islam2021broad} and \citet{chen2022perfectly} empirically investigate the transferability of the joint training of SimCLR ad SupCon, whereas ours focuses on theoretically analyzing the performance of their joint training under linear probing. We discuss the connections and differences as follows.

\citet{islam2021broad} finds that the combination of SimCLR and SupCon significantly improves transfer learning performance over Cross-Entropy. \citet{chen2022perfectly} investigates the problems of coarse-to-fine transfer learning by adding a weighted class-conditional InfoNCE loss and a class-conditional autoencoder to SupCon. 
Similar to our results, according to Tables 2 and 3 of \citet{islam2021broad} and Table 3 of \citet{chen2022perfectly}, the transfer learning performance of the combination of SupCon and InfoNCE (or class-conditional InfoNCE) is comparable to but has no significant improvement over the winner of SupCon and SimCLR. 
In addition, the robustness investigated in these two papers is conceptually different from ours. The robustness in \citet{islam2021broad} is to \textit{image corruptions} and the robustness in \citet{chen2022perfectly} measures \textit{how well an algorithm can recover hidden subgroups in an unsupervised setting}, whereas our manuscript investigates the robustness to \textit{label corruptions} (label noise).

\subsection{Transfer Learning Performances}
\label{app::transfer}

We conduct additional experiments to evaluate the representation quality with the transfer learning performance on downstream tasks including detection, segmentation, and fine-tuning. We evaluate transfer learning performance of models pretrained on TinyImageNet-200 with SimCLR, SupCon, and Mix. All models are pretrained under the settings of Section \ref{sec::exp_setting}. For Mix, we set $\theta=0.2$. For SupCon and Mix, the models are pretrained under noise rate $\gamma \in \{20\%, 40\%, 60\%, 80\%\}$. We list the results as follows. The best results are marked in \textbf{bold} and the second best marked in \underline{underline}.

\textbf{Detection.}
Object detection is fine-tuned on PASCAL VOC 07+12 dataset. The detector is Faster R-CNN. All models are fine-tuned for 12 epochs. We evaluate the models by the default VOC metric of AP50.

\begin{table}[!h]
	\centering
	\caption{Performance comparisons of object detection.}
	\label{tab::det}
	\begin{tabular}{ccccc}
		\toprule
		$\gamma$ & 20\% & 40\% & 60\% & 80\% \\
		\midrule
		SimCLR & 65.3 & 65.3 & \underline{65.3} & \underline{65.3} \\
		SupCon & \textbf{68.5} & \textbf{66.9} & \textbf{66.2} & \textbf{67.0} \\
		Mix & \underline{66.9} & \underline{65.5} & 63.8 & 63.2 \\
		\bottomrule
	\end{tabular}
\end{table}

\textbf{Segmentation.}
Segmentation is fine-tuned on Pascal VOC 12 dataset. The model is DeeplabV3. All models are fine-tuned for 20,000 iterations. We evaluate the models by mean IoU.

\begin{table}[!h]
	\centering
	\caption{Performance comparisons of segmentation.}
	\label{tab::seg}
	\begin{tabular}{ccccc}
		\toprule
		$\gamma$ & 20\% & 40\% & 60\% & 80\% \\
		\midrule
		SimCLR & 36.61 & 36.61 & \underline{36.61} & \underline{36.61} \\
		SupCon & \textbf{49.07} & \textbf{47.58} & \textbf{46.57} & \textbf{42.47} \\
		Mix & \underline{40.33} & \underline{38.14} & 33.66 & 34.33 \\
		\bottomrule
	\end{tabular}
\end{table}

\textbf{Fine-tuning.}
We finetune the pretrained models on the labeled CIFAR-10 data. The weights of the linear classifier used to fine-tune the encoder network are initialized to zero. On CIFAR-10, models are fine-tuned for 90 epochs. All results are reported on the standard CIFAR-10 test set and are summarized in the following table.

\begin{table}[!h]
	\centering
	\caption{Performance comparisons of fine-tuning.}
	\label{tab::fine}
	\begin{tabular}{ccccc}
		\toprule
		$\gamma$ & 20\% & 40\% & 60\% & 80\% \\
		\midrule
		SimCLR & 68.01 & \textbf{68.01} & \textbf{68.01} & \textbf{68.01} \\
		SupCon & \textbf{71.35} & \underline{67.46} & 61.62 & 49.03 \\
		Mix & \underline{70.01} & 66.89 & \underline{63.86} & \underline{55.38} \\
		\bottomrule
	\end{tabular}
\end{table}

As shown in the Tables \ref{tab::det}, \ref{tab::seg}, and \ref{tab::fine}, the performance of joint training (Mix) is no better than the winner of SimCLR and InfoNCE across all noise rates, which is similar to the conclusion on linear probing discussed in Section \ref{sec::exp_baseline}.

\end{document}